\documentclass[twoside,11pt]{article}

\usepackage{jmlr2e}

\usepackage[utf8]{inputenc} 
\usepackage[T1]{fontenc}
\usepackage{url}
\usepackage{hyperref}       
\usepackage{booktabs}       
\usepackage{amsfonts}       
\usepackage{nicefrac}       
\usepackage{microtype}      
\usepackage{ifthen}
\usepackage[linesnumbered,ruled,algo2e]{algorithm2e}
\usepackage{amsthm, amssymb, amsmath}
\usepackage{graphicx}
\usepackage{cleveref}

\usepackage[disable]{todonotes}

\usepackage{epstopdf}
\usepackage{algorithm}
\usepackage{algorithmic}
\usepackage{wrapfig}
\usepackage{subcaption}
\usepackage{bbm}

\usepackage{xspace}
\usepackage{enumitem,kantlipsum}
\usepackage{subcaption}
\usepackage{bigints}
\usepackage{float}

\newcommand{\totalPulls}{T}
\newcommand{\setofArms}{\mathcal{K}}
\newcommand{\competitiveArms}{\mathcal{C}}
\newcommand{\numArms}{K}
\newcommand{\arm}{k}
\newcommand{\slot}{t}
\newcommand{\pulls}{n}
\newcommand{\gap}{\Delta} 
\newcommand{\reward}{R}
\newcommand{\regret}{Reg}
\newcommand{\bestarm}{{k^*}}
\newcommand{\maxArm}{{k^{\text{max}}}}
\newcommand{\Index}{I}
\newcommand{\estimateMean}{\hat{\mu}}
\newcommand{\indicator}{\mathbbm{1}}

\newcommand{\empCompetitive}{$\hat{\Theta}_\slot$-Competitive\xspace}
\newcommand{\empNonCompetitive}{$\hat{\Theta}_\slot$-Non-competitive\xspace}

\def \OO {\mathrm{O}}
\def \oo {\mathrm{o}}

\newtheorem{lem}{Lemma}
\newtheorem{thm}{Theorem}
\newtheorem{defn}{Definition}

\newtheorem{prop}{Proposition}

\newtheorem{rem}{Remark}
\newtheorem{fact}{Fact}

\newcommand{\E}[1]{\mathbb{E}\left[{#1}\right]}

\DeclareMathOperator*{\argmax}{arg\,max}

\graphicspath{{Figures/}}

\crefname{equation}{}{}
\Crefname{equation}{}{}
\crefname{thm}{theorem}{theorems}
\Crefname{thm}{Theorem}{Theorems}
\crefname{clm}{claim}{claims}
\Crefname{clm}{Claim}{Claims}
\Crefname{coro}{Corollary}{Corollaries}
\Crefname{lem}{Lemma}{Lemmas}
\Crefname{sec}{Section}{Sections}
\crefname{app}{appendix}{appendices}
\Crefname{app}{Appendix}{Appendices}
\crefname{prop}{proposition}{propositions}
\Crefname{prop}{Proposition}{Propositions}
\Crefname{propty}{Property}{Properties}
\crefname{figure}{figure}{figures}
\Crefname{figure}{Figure}{Figures}
\crefname{fig}{figure}{figures}
\Crefname{fig}{Figure}{Figures}
\crefname{defn}{definition}{definitions}
\Crefname{defn}{Definition}{Definitions}
\crefname{fact}{fact}{facts}
\Crefname{fact}{Fact}{Facts}
\crefname{appendix}{appendix}{appendices}
\Crefname{appendix}{Appendix}{Appendices}
\crefname{algo}{algorithm}{algorithms}
\Crefname{algo}{Algorithm}{Algorithms}
\crefname{algorithm}{algorithm}{algorithms}
\Crefname{algorithm}{Algorithm}{Algorithms}
\crefname{conj}{conjecture}{conjectures}
\Crefname{conj}{Conjecture}{Conjectures}
\crefname{obs}{observation}{observations}
\Crefname{obs}{Observation}{Observations}

\begin{document}

\title{A Unified Approach to Translate Classical Bandit Algorithms to the Structured Bandit Setting}

\author{\name Samarth Gupta \email samarthg@andrew.cmu.edu \\
 \addr Carnegie Mellon University\\
 Pittsburgh, PA 15213 
 \AND
 \name Shreyas Chaudhari \email schaudh2@andrew.cmu.edu \\
 \addr Carnegie Mellon University\\
 Pittsburgh, PA 15213 
 \AND
 \name Subhojyoti Mukherjee \email smukherjee27@wisc.edu \\
 \addr University of Wisconsin-Madison\\
 Madison, WI 53706
 \AND
 \name Gauri Joshi \email gaurij@andrew.cmu.edu \\
 \addr Carnegie Mellon University\\
 Pittsburgh, PA 15213 
 \AND
 \name Osman Ya\u{g}an \email oyagan@andrew.cmu.edu\\
 \addr Carnegie Mellon University\\
 Pittsburgh, PA 15213}

\editor{No editors}
\maketitle

\begin{abstract}

We consider a finite-armed structured bandit problem in which mean rewards of different arms are known functions of a common hidden parameter $\theta^*$. Since we do not place any restrictions on these functions, the problem setting subsumes several previously studied frameworks that assume linear or invertible reward functions. We propose a novel approach to gradually estimate the hidden $\theta^*$ and use the estimate together with the mean reward functions to substantially reduce exploration of sub-optimal arms. This approach enables us to fundamentally generalize any classical bandit algorithm including UCB and Thompson Sampling to the structured bandit setting. We prove via regret analysis that our proposed UCB-C algorithm (structured bandit versions of UCB) pulls only a {\em subset} of the sub-optimal arms $\OO(\log T)$ times while the other sub-optimal arms (referred to as \emph{non-competitive} arms) are pulled  $\OO(1)$ times. As a result, in cases where all sub-optimal arms are non-competitive, which can happen in many practical scenarios, the proposed algorithm achieves bounded regret. We also conduct simulations on the {\sc Movielens} recommendations dataset to demonstrate the improvement of the proposed algorithms over existing structured bandit algorithms.

\end{abstract}


\section{Introduction}
\label{sec:AltIntro}

\subsection{Overview}
The Multi-armed bandit problem \citep{lai1985asymptotically} (MAB) falls under the umbrella of sequential decision-making problems. It has numerous applications such as clinical trials \citep{villar2015multi}, system testing \citep{tekin2017multi}, scheduling in computing systems \citep{mora2009stochastic}, and web optimization \citep{white2012bandit, sen2017identifying}, to name a few. In the classical $K$-armed bandit formulation, a player is presented with $K$ arms. At each time step $t=1,2, \ldots$, she decides to pull an arm $\arm \in \setofArms$ and receives a random {\em reward} $R_{k}$ with unknown mean $\mu_\arm$. The goal of the player is to maximize their expected cumulative reward (or equivalently, minimize expected cumulative {\em regret}) over $T$ time steps. In order to do so, the player must strike a balance between estimating the unknown rewards {\em accurately} by pulling all the arms (exploration) and always pulling the current best arm (exploitation). The seminal work of Lai and Robins (1985) proposed the  Upper Confidence Bound (UCB) algorithm that balances the exploration-exploitation tradeoff in the MAB problem. Subsequently, several algorithms such as UCB1 \citep{auer2002finite}, Thompson Sampling (TS) \citep{agrawal2012analysis} and KL-UCB \citep{garivier2011kl} were proposed and analyzed for the classical MAB setting.

In this paper, we study a fundamental variant of classical multi-armed bandits called
the {\em structured multi-armed bandit problem},
where mean rewards of the arms are functions of a {\em hidden} parameter $\theta$. That is, the expected reward $\E{R_{k}|\theta} = \mu_k(\theta)$ of arm $k$ is a {\em known} function of the parameter $\theta$ that lies in a ({\em known}) set $\Theta$. However, the true value of $\theta$, denoted as $\theta^*$, is unknown. The dependence of mean rewards on the common parameter introduces a {\em structure} in the MAB problem. For example, rewards observed from an arm may provide partial information  about the mean rewards of other arms, 
making it possible to significantly lower the resulting cumulative regret as compared to the classical MAB setting.

Structured bandit models arise in many applications and have been studied by several authors with motivating applications including dynamic pricing (described in \cite{ata2015global}), cellular coverage optimization (by \cite{shen2018generalized}), drug dosage optimization (discussed in \cite{wang2018regional}) and system diagnosis; see Section \ref{subsec:illustrative} for an illustrative application of the structured MAB framework. In this paper,  we consider a  {\em general} version of the structured MAB framework that subsumes and generalizes several previously considered settings. More importantly, we propose a novel and unified approach that would allow extending any current or future MAB algorithm (UCB, TS, KL-UCB, etc.) to the structured setting; see Sections \ref{subsec:intro_contributions} and \ref{subsec:relatedwork} for our main contributions and a comparison of our work with related literature. 

\subsection{An Illustrative Example}
\label{subsec:illustrative}

\begin{figure}[t]
    \centering
    \includegraphics[width = 0.6\textwidth]{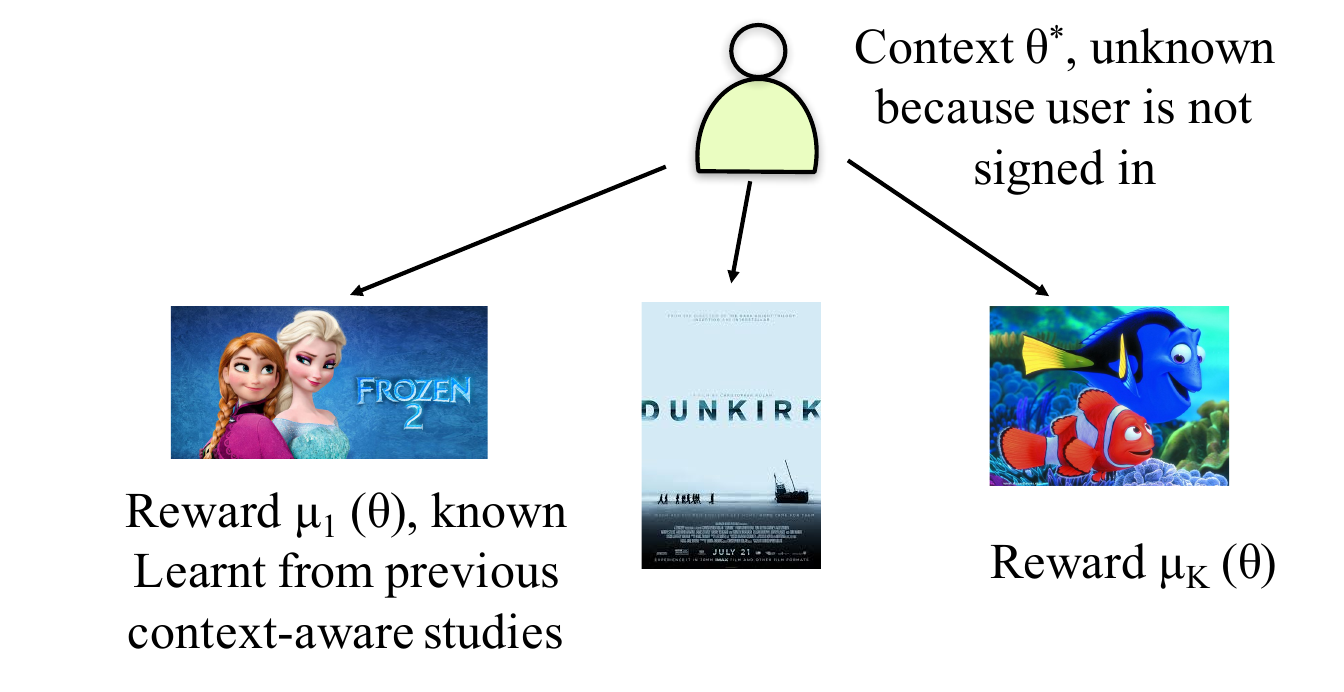}
    \caption{\sl Movie recommendation application of the structured bandit framework studied in this paper. The context $\theta$ (for example, the age of the user) is unknown because the user is not signed in. But if a user gives a high rating the first movie (Frozen) one could infer that the age $\theta$ is small, which in turn implies that the user will give a high rating to the third movie (Finding Nemo).}
    \label{fig:movie_reco_eg}
\end{figure}

For illustration purposes, consider the example of movie recommendation, where a company would like to decide which movie(s) to recommend to each user with the
goal of maximizing user engagement (e.g., in terms of click probability and time spent watching, etc.). In order to achieve this, the company needs to identify the most appealing movie for the user in an online manner and this is where multi-armed bandit algorithms can be helpful. However, classical MAB algorithms  are typically based on the (implicit) assumption that rewards from different arms (i.e., different movies in this context) are independent of each other. This assumption is unlikely to hold in reality since the user choices corresponding to different movies are likely to be related to each other; e.g., the engagement  corresponding to different movies may depend on the age/occupation/income/taste of the user.

To address this, {\em contextual} bandits \citep{li2010contextual, sen2018contextual} have been proposed and studied widely for personalized recommendations. There, it is assumed that before making a choice (of which movie to recommend), a
 {\em context} feature 
of the user is observed; the context can include personal information of the user including age, occupation, income, or previous browsing information. 
Contextual bandit algorithms aim to learn the mapping from context information to the most appealing arm, and 
can prove useful in applications involving personalized recommendations (or, advertising). 
However in several use cases, observing contextual features leads to privacy concerns. In addition, the contextual features may not be visible for {\em new} users or users who are signed in anonymously to protect their identity.

The structured bandit setting considered in this paper (and by many others \cite{ata2015global,shen2018generalized,wang2018regional}) 
can be viewed as 
the same problem setting with contextual bandits with the following difference. Unlike contextual bandits, the context of the users are
{\em hidden} in the structured setting, but in return it is assumed that the mappings from the contexts to {\em mean} rewards of arms are known a priori. 
It is anticipated that the mean reward mappings can be learned from paid surveys in which users participate with their consent.  
The proposed structured bandit framework's goal is to use this information to provide the best recommendation to an anonymous user whose context vector $\theta$ is unknown;
e.g., see \Cref{fig:movie_reco_eg}.
\emph{Thus, our problem formulation is complementary to contextual bandits; in contextual bandits the context $\theta$ is known while the reward mappings $\mu_k(\theta)$ are unknown, whereas in our setting $\theta$ is unknown and the mean rewards $\mu_k(\theta)$ are known.} A detailed problem formulation discussing the assumptions and extensions of this set-up is given in \Cref{sec:model}.

\subsection{Main Contributions and Organization}
\label{subsec:intro_contributions}
We summarize the key contributions of the paper below. The upcoming sections will develop each of these in detail.

\begin{enumerate}[wide, labelwidth=!, labelindent=0pt]
\item \textbf{General Setting Subsuming Previous Works}: Structured bandits have been studied in the past \citep{ata2015global,wang2018regional,mersereau2009structured,filippi2010parametric,lattimore2016end, graves1997asymptotically}
but with certain restrictions (e.g., being linear, invertible, etc.)  on the mean reward mappings $\mu_k(\theta)$.
We consider a general setting that puts no restrictions on the mean reward mappings. In fact, our setting subsumes recently studied models such as Global Bandits \citep{ata2015global}, Regional Bandits \citep{wang2018regional} and structured bandits with linear rewards \citep{mersereau2009structured}; see \Cref{sec:priorWork} for a detailed comparison with previous works. 

There are a couple of recent works \citep{lattimore2014bounded,combes2017minimal} that do consider a general structured bandit setting similar to our work  -- see \Cref{sec:priorWork} for details. Our approach differs from these in its flexibility to extend any classical bandit algorithm (UCB, Thompson sampling, etc.) to the structured bandit setting. In particular, using Thompson sampling \citep{thompson1933likelihood, agrawal2013thompson} as the underlying bandit algorithm yields a robust and empirically superior way (see \Cref{subsec:sim}) to minimize superfluous exploration. The  UCB-S algorithm proposed in \citep{lattimore2014bounded} extends the UCB algorithm to structured setting. However, the approach presented in  
\citep{lattimore2014bounded}
can not be extended to Thompson sampling or other classical bandit algorithms; in fact, this point was highlighted in \citep{lattimore2014bounded} as an open question for future work. In \citep{combes2017minimal}, there are several assumptions in the model that are not imposed here, including the assumption that the conditional reward distributions are known and reward mappings are continuous. In addition, the main focus of \citep{combes2017minimal} is the parameter regime where regret scales logarithmically with time $T$, while our approach demonstrates the possibility of achieving {\em bounded} regret.

\item \textbf{Extending any classical bandit algorithm to the structured bandit setting}: We propose a novel and unified approach that would allow extending any classical or future MAB algorithm (that is developed for  the non-structured setting) to the structured bandit framework
given in \Cref{fig:problem_setup}.
Put differently, we propose a {\em class} of structured bandit algorithms referred to as \textsc{Algorithm}-C, where ``\textsc{Algorithm}" can be any classical bandit algorithm including UCB, TS, KL-UCB, etc. 
A detailed description of the
resulting algorithms, e.g., UCB-C, TS-C, etc.,
are given in \Cref{sec:algorithmC} with their steps  illustrated in \Cref{fig:stepbystep}. 

\begin{figure}[t]
    \centering
    \includegraphics[width = 0.6\textwidth]{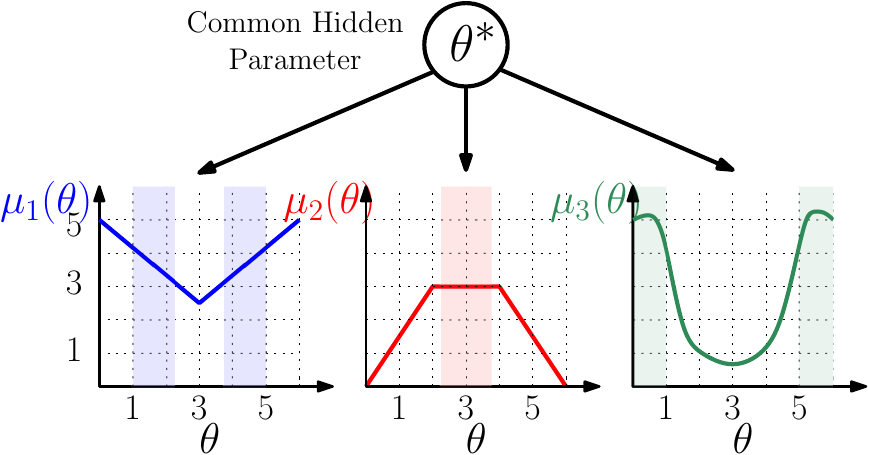}
    \caption{\sl Structured bandit setup: mean rewards of different arms share a common hidden parameter. This example illustrates a 3-armed bandit problem with shaded regions indicating the values of $\theta$ for which the particular arm is optimal.} 
    \label{fig:problem_setup}
\end{figure}

\item \textbf{Unified regret analysis}: 
A key benefit of our algorithms is that they pull a subset of the arms (referred to as the {\em non-competitive} arms) only $\OO(1)$ times.
Intuitively, an arm is non-competitive if it can be identified as sub-optimal with high probability using only the samples from the optimal arm $k^*$.
This is in contrast to classical MAB algorithms where all sub-optimal arms are pulled $\OO(\log T)$ times, where $T$ is the total number of rounds. This is shown by analyzing the expected regret $\E{Reg(\totalPulls)}$, which is the difference between the expected cumulative reward obtained by using the proposed algorithm and the expected cumulative reward of a genie policy that always pulls the optimal arm $k^*$. 
In particular, we provide rigorous regret analysis for UCB-C as summarized in the theorem below, and describe how our proof technique can be extended to other classical MAB algorithms. 
\begin{thm}[Expected Regret Scaling]
The expected regret of the UCB-C algorithm has the following scaling with respect to the number of rounds $T$:
\begin{align}
\E{Reg(\totalPulls)} &\leq (C(\theta^*)-1) \cdot \OO(\log \totalPulls) + (\numArms - C(\theta^*)) \OO(1) \label{eqn:regret_bnd}
\end{align}
where $C(\theta^*)$ is the number of competitive arms (including the optimal arm $k^*$) and $\theta^*$ is the true value of the hidden parameter. The remaining $K-C(\theta^*)$ arms are called non-competitive. An arm is said to be non-competitive if there exists an $\epsilon > 0$ such that $\mu_k(\theta) < \mu_{k^*}(\theta)$ for all $\theta \in \Theta^{*(\epsilon)}$, where $\Theta^{*(\epsilon)} = \{ \theta \in \Theta: |\mu_{k^*}(\theta) - \mu_{k^*}(\theta^*)| \leq \epsilon\}$ (more details in \Cref{sec:regret}). 
\label{thm:regret_bnd}
\end{thm}

The exact regret upper bound with all the constants follows from \Cref{thm:nonComp} and \Cref{thm:CompUCB} in \Cref{sec:regret}. Recall that for the standard MAB setting \citep{lai1985asymptotically}, the regret upper bound is $(K-1)\OO(\log T)$, where $K$ is the total number of arms. \Cref{thm:regret_bnd} reveals that with our algorithms only $C(\theta^*) - 1$ out of the $K-1$ sub-optimal arms are pulled $\OO(\log T)$ times. The other arms are pulled only $\OO(1)$ times. 

\item \textbf{Reduction in the effective number of arms.}
{For 
any given set of reward functions $\mu_k(\theta)$,
the number $C(\theta^*)$ of competitive arms depends on the {\em unknown} parameter $\theta^*$;  see \Cref{fig:C_theta_illustration} in \Cref{sec:regret} for an illustration of this fact. We show that $C(\theta^*)$ can be much smaller than $K$ in many practical cases. This is because, the reward functions (particularly that corresponding to the {\em optimal} arm) can provide enough information about the hidden $\theta^*$, which in turn can help infer the sub-optimality of several other arms. More specifically, this happens when the reward functions are not flat around $\theta^*$, that is, the pre-image set of \{$\theta \in \Theta: \mu_k(\theta) =\mu_k(\theta^*) \}$ is small. In the special case where the optimal arm $\bestarm$ is {\em invertible} or has a unique maximum at $\mu_{k^*}(\theta^*)$, $C(\theta^*) = 1$ and our algorithms can achieve $\OO(1)$ regret.} 

\item \textbf{Evaluation on real-world datasets}: In \Cref{sec:regret}, we present extensive simulations comparing the regret of the proposed algorithm with previous methods such as GLM-UCB \citep{filippi2010parametric} and UCB-S \citep{lattimore2014bounded}. We also present simulation results for the case where the hidden parameter $\theta$ is a {\em vector}. In \Cref{sec:experiments}, we perform experiments on the {\sc Movielens} dataset to demonstrate the applicability of the UCB-C and TS-C algorithms. Our experimental results show that both UCB-C and TS-C lead to significant improvement over the performance of existing bandit strategies. In particular, TS-C is shown to consistently outperform all other algorithms across a wide range of settings. 

\end{enumerate}

\section{Problem Formulation}
\label{sec:model}
Consider a multi-armed bandit setting with the set of arms $\mathcal{K} = \{1,2, \ldots, \numArms\}$. 
At each round $\slot$, the player pulls arm $\arm_\slot \in \mathcal{K}$ and observes a reward $\reward_{k_\slot}$. The reward $\reward_{k_\slot}$ is a random variable with mean $\mu_{\arm_\slot}(\theta) = \E{\reward_{k_\slot} | \theta, \arm_{\slot}}$, where $\theta$ is a \emph{fixed, but unknown parameter} which lies in a known set $\Theta$; see \Cref{fig:problem_setup}. 

We denote the (unknown) true value of  $\theta$ by $\theta^*$. There are no restrictions on the set $\Theta$. Although we focus on scalar $\theta$ in this paper for brevity, the proposed algorithms and regret analysis can be  generalized to the case where we have a hidden parameter {\em vector} $\vec{\theta} = [\theta_1 , \theta_2, \dots \theta_m]$. In \Cref{sec:regret}, we present simulation results for the case of a parameter vector $\mathbf{\theta}$. The mean reward functions $\mu_k(\theta) = \mathbb{E}[R_{k}|\theta]$ for $k \in \mathcal{K}$ can be arbitrary  functions of $\theta$
with no linearity or continuity constraints imposed.
While $\mu_k(\theta)$ are known to the player, the conditional distribution of rewards, i.e., $p(\reward_{k}|\theta)$ is not known.

We assume that the rewards $R_{\arm}$ are sub-Gaussian with variance proxy $\sigma^2$, i.e., \\ $\E{\exp\left(s\left(R_{k} - \E{R_{k}}\right)\right)} \leq \exp\left(\frac{\sigma^2 s^2}{2}\right)$ $\forall{s \in \mathbb{R}}$, and $\sigma$ is known to the player. Both assumptions are common in the multi-armed bandit literature \citep{bubeck2012regret, jamieson2014best, ata2015global, lattimore2014bounded}.
In particular, the sub-Gaussianity of rewards enables us to apply Hoeffding's inequality, which is essential for the analysis of regret (defined below).

The objective of the player is to select arm $\arm_\slot$ in round $t$ so as to maximize her cumulative reward $\sum_{\slot=1}^{T} \reward_{k_\slot}$ after $T$ rounds. If the player had known the hidden $\theta^*$, then she would always pull arm $\bestarm = \arg \max_{\arm \in \setofArms} \mu_\arm(\theta^*)$ that yields the highest mean reward at $\theta = \theta^*$. We refer to $\bestarm$
as the {\em optimal} arm. Maximizing the cumulative reward is equivalent to minimizing the \emph{cumulative regret}, which is defined as
\begin{align*}
    \regret(\totalPulls) \triangleq \sum_{\slot = 1}^{\totalPulls} \left(\mu_{\bestarm}(\theta^*) - \mu_{\arm_\slot}(\theta^*) \right)
    &= \sum_{\arm \neq \bestarm} \pulls_\arm(\totalPulls) \gap_{\arm},
\end{align*}
where $\pulls_\arm(\totalPulls)$ is the number of times arm $\arm$ is pulled in $\totalPulls$ slots and  $\gap_\arm \triangleq \mu_{\bestarm}(\theta^*) - \mu_{\arm}(\theta^*)$ is the {\em sub-optimality gap} of arm $\arm$. 
Minimizing the cumulative regret is in turn equivalent to minimizing $\pulls_\arm(\totalPulls)$, the number of times each sub-optimal arm $\arm \neq \arm^*$ is pulled. 

\begin{rem}[Connection to classical Multi-armed Bandits] The classical multi-armed bandit setting, which does not explicitly consider a {\em structure} among the mean rewards of different arms, is a special case of the proposed structured bandit framework. It corresponds to having a hidden parameter vector $\vec{\theta} = (\theta_1, \theta_2, \ldots, \theta_\numArms)$ and the mean reward of each arm being $\mu_\arm = \theta_\arm$. In fact, our proposed algorithm described in \Cref{sec:algorithmC} reduces to standard UCB or Thompson sampling (\cite{lai1985asymptotically, auer2002finite}) in this special case.
\end{rem}

The proposed structured bandit subsumes several previously considered models where the rewards are assumed to be linear \citep{mersereau2009structured, lattimore2016end}, invertible and H\"{o}lder continuous \citep{ata2015global, wang2018regional}, etc. See \Cref{sec:priorWork} for a detailed comparison with these works.

\section{Related Work}
\label{sec:priorWork}
Since we do not make any assumptions on the mean reward functions $\mu_1(\theta), \mu_2(\theta), \ldots, \mu_\numArms(\theta)$, our model subsumes several previously studied frameworks \citep{ata2015global,wang2018regional,mersereau2009structured}. The similarities and differences between our model and existing works are discussed below.

\label{subsec:relatedwork}
\noindent

\vspace{0.1cm} \noindent \textbf{Structured bandits with linear rewards \citep{mersereau2009structured}.} In \citep{mersereau2009structured}, the authors consider a similar model with a common hidden parameter $\theta \in \mathbb{R}$, but the mean reward functions, $\mu_k(\theta)$ are linear in $\theta$. Under this assumption, they design a greedy policy that achieves bounded regret. Our formulation  does not make linearity assumptions on the reward functions. In the special case when $\mu_k(\theta)$ are linear, our proposed algorithm also achieves bounded regret.

\noindent
\textbf{Global and regional bandits.} The papers \citep{ata2015global, wang2018regional} generalize this to invertible and H\"{o}lder-continuous reward functions. Instead of scalar $\theta$, \citep{wang2018regional} considers $M$ common unknown parameters, that is, $\theta = (\theta_1, \theta_2, \ldots, \theta_M)$. Under these assumptions, \citep{ ata2015global, wang2018regional} demonstrate that it is possible to achieve bounded regret through a greedy policy. In contrast, our work makes no invertibility or continuity assumptions on the reward functions $\mu_k(\theta)$. In the special case when $\mu_k(\theta)$ are invertible, our proposed algorithm also achieves bounded regret.

\vspace{0.1cm} \noindent
\textbf{Finite-armed generalized linear bandits.}
In the finite-armed linear bandit setting \citep{lattimore2016end}, the reward function of arm $x_k$ is $\vec{\theta}^\intercal x_k$, which is subsumed by our formulation. For the case when $\mu_\arm(\theta) = g(\vec{\theta}^\intercal x_k)$, our setting becomes the generalized linear bandit setting \citep{filippi2010parametric}, for some known function $g$. Here, $\theta$ is the shared unknown parameter.
Due to the particular form of the mean reward functions, linear bandit algorithms construct confidence ellipsoid for $\theta^*$ to make decisions. This approach cannot be easily extended to non-linear settings. Although designed for the more general non-linear setting, our algorithms demonstrate comparable regret to the GLM-UCB \citep{filippi2010parametric}, which is designed for linear bandits.  

\vspace{0.1cm} \noindent
\textbf{Minimal exploration in structured bandits \citep{combes2017minimal}}
The problem formulation in \citep{combes2017minimal} is very similar to this paper. However, \citep{combes2017minimal} assumes knowledge of the conditional reward distribution $p(R_{k}|\theta)$ in addition to knowing the mean reward functions $\mu_k(\theta)$. It also assumes that the mappings $\theta \rightarrow \mu_\arm(\theta)$ are continuous.
As noted before, none of these assumptions are imposed in this paper. Another major difference of \citep{combes2017minimal} with our work is that they focus on obtaining asymptotically optimal results for the regimes where regret scales as $\log(T)$. When all arms are \textit{non-competitive} (the case where our algorithms lead to $\OO(1)$ regret), the solution to the optimization problem described in \citep[Theorem 1]{combes2017minimal} becomes $0$, causing
the algorithm to get stuck in the exploitation phase. Put differently, the algorithm proposed in \citep{combes2017minimal} is not applicable to cases where $C(\theta^*)=1$. An important contribution of \citep{combes2017minimal} is that it provides a lower bound on the regret of structured bandit algorithms. In fact, the lower bound presented in this paper is directly based on the lower bound in \citep{combes2017minimal}.

\noindent
\textbf{Finite-armed structured bandits \citep{lattimore2014bounded}.} The work closest to ours is \citep{lattimore2014bounded}. They consider the same model that we consider and propose the UCB-S algorithm, which is a UCB-style algorithm for this setting. Our approach allows us to extend our UCB-style algorithm to other classical bandit algorithms such as Thompson sampling. In \Cref{sec:regret} and \Cref{sec:experiments}, we extensively compare our proposed algorithms (both qualitatively and empirically) with the UCB-S algorithm proposed in \citep{lattimore2014bounded}. As observed in the simulations, UCB-S is susceptible to small changes in the mean reward functions and $\theta^*$, whereas the UCB-C algorithm that we propose here is seen to be much more robust to such variations. 

\vspace{0.1cm} 
\noindent \textbf{Connection to information-directed sampling.} Works such as \citep{russo2014learning, gopalan2014thompson} consider a similar structured setting but assume that the conditional reward distributions $p(R_k | \theta)$ and the prior $p(\theta)$ are known, whereas we only consider that the {\em mean} reward functions $\mu_k(\theta) = \E{R_k| \theta}$ are known. The proposed algorithms are based on Thompson sampling from the posterior distribution of $\theta$. Firstly, this approach will require a good prior over $\theta$, and secondly, updating the posterior can be computationally expensive since it requires computing integrals over possibly high-dimensional spaces. The focus of \citep{russo2014learning} is on {\em worst-case} regret bounds (which are typically $\OO(\sqrt{T})$), where the minimum gap between two arms can scale as $O(\log T/T)$, while \citep{gopalan2014thompson} gives gap-dependent regret bounds in regimes where the regret scales as $\OO(\log{T})$. In addition to providing gap-dependent regret bounds, we also identify regimes where it is possible to achieve $\OO(1)$ regret. 

\vspace{0.1cm} \noindent \textbf{Best-arm identification.} In several applications such as as hyper-parameter optimization \citep{li2017hyperband} and crowd-sourced ranking \citep{tanczos2017kl, jamieson2015next, heckel2016active}, the objective is to maximize the probability of identifying the  arm with the highest expected reward within a given time budget of $T$ slots instead of maximizing the cumulative reward; that is, the focus is on exploration rather than exploitation. Best-arm identification started to be studied fairly recently \citep{bubeck2009pure,audibert2010best,jamieson2014lil}. A variant of the fixed-time budget setting is the fixed-confidence setting \citep{jamieson2014best, kaufmann2016complexity, mannor2004sample, garivier2016optimal, gabillon2012best}, where the aim is to minimize the number of slots required to reach a $\delta$-error in identifying the best arm. Very few best-arm identification works consider structured rewards \citep{sen2017identifying, soare2014best, huang2017structured, tao2018best}, and they mostly assume {\em linear} rewards. The algorithm design and analysis tools are quite different in the best-armed bandit identification problem as compared to regret minimization. Thus, extending our approach to best-arm identification would be a non-trivial future research direction.

\section{PROPOSED ALGORITHM: \textsc{Algorithm}-C}
\label{sec:algorithmC}
For the problem formulation described in \Cref{sec:model}, we propose the following three-step algorithm called \textsc{Algorithm}-Competitive, or, in short, \textsc{Algorithm}-C. \Cref{fig:stepbystep} illustrates the algorithm steps for the mean reward functions shown in \Cref{fig:problem_setup}. Step 3 can employ any classical multi-armed bandit algorithm such as UCB or Thompson Sampling (TS), which we denote by \textsc{Algorithm}. Thus, we give a unified approach to translate any classical bandit algorithm to the structured bandit setting. The formal description of \textsc{Algorithm}-C with UCB and TS as final steps is given in \Cref{alg:formalAlgoUCBC} and \Cref{alg:formalAlgoTSC}, respectively.

\begin{figure*}[t]
    \centering
    \includegraphics[width = 0.9\textwidth]{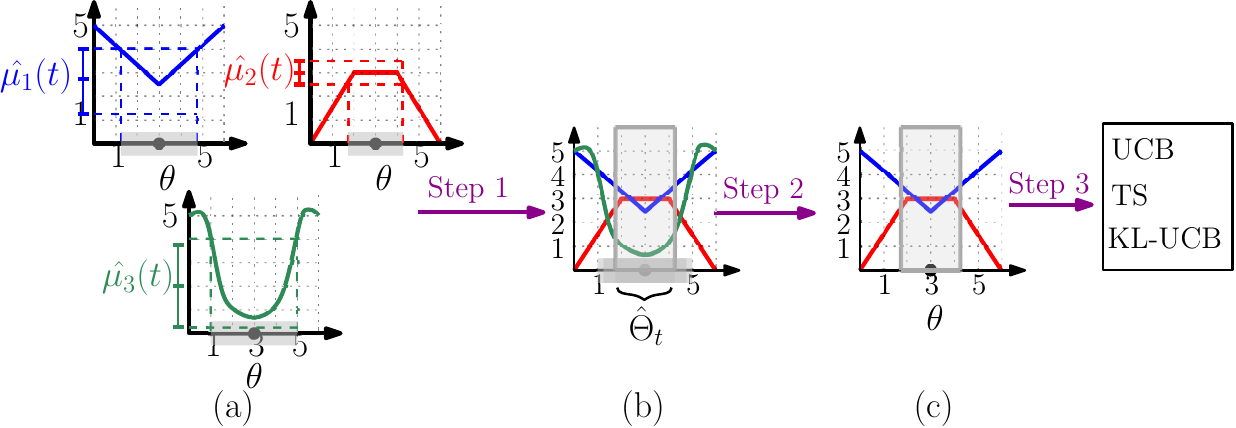}
    \caption{ Illustration of the steps of the proposed algorithm. In step 1, for each arm $\arm$ we find the set of $\theta$ such that $|\mu_\arm(\theta) - \hat{\mu}_\arm(t) | < \sqrt{2 \alpha \sigma^2 \log \slot/\pulls_\arm(\slot)}$ (shaded in gray in part (a)). The intersection of these sets gives the confidence set $\hat{\Theta}_\slot$ shown in part (b). Next, we observe that the mean reward $\mu_3(\theta)$ of Arm 3 (shown in green) cannot be optimal if the unknown parameter $\theta^*$ lies in set $\hat{\Theta}_\slot$. Thus, it is declared as $\hat{\Theta}_\slot$-non-competitive and not considered in Step 3. In step 3, we pull one of the $\hat{\Theta}_\slot$-competitive arms (shown in (c)) using a classical bandit algorithm such as UCB, Thompson Sampling, KL-UCB, etc.
    }
    \label{fig:stepbystep}
\end{figure*}

At each round $t+1$, the algorithm performs the following steps:

\vspace{0.1cm} \noindent
\textbf{Step 1: Constructing a confidence set, $\hat{\Theta}_t$.} From the samples observed till round $\slot$, we define the confidence set as follows: 
\begin{equation}
    \small{\hat{\Theta}_t = \left\{\theta: \forall{\arm \in \setofArms}, \quad | \mu_\arm(\theta) - \hat{\mu}_\arm(t) | < \sqrt{\frac{2 \alpha \sigma^2 \log \slot}{\pulls_\arm(\slot)}}\right\}.} \nonumber
\end{equation}

Here, $\hat{\mu}_k(t)$ is the empirical mean of rewards obtained from the $\pulls_k(\slot)$ pulls of arm $k$. For each arm $k$, we construct a confidence set of $\theta$ such that the true mean $\mu_k(\theta)$ is within an interval of size $\sqrt{\frac{2 \alpha \sigma^2 \log \slot}{\pulls_\arm(\slot)}}$  from $\hat{\mu}_k(t)$. This is illustrated by the error bars along the y-axis in \Cref{fig:stepbystep}(a), with the corresponding confidence sets shown in grey for each arm. 
Taking the intersection of these $K$ confidence sets gives us  $\hat{\Theta}_t$, wherein $\theta$ lies with high probability, as shown in \Cref{fig:stepbystep}(b).

\vspace{0.1cm} \noindent
\textbf{Step 2: Finding the set $\competitiveArms_t$ of \empCompetitive arms.} We let  $\competitiveArms_t$ denote the set of \empCompetitive arms at round $\slot$,
defined as follows.

\begin{defn}[\empCompetitive arm]
\label{defn:competitive}
An arm $\arm$ is said to be \empCompetitive if its mean reward is the highest among all arms for some $\theta \in \hat{\Theta}_{t}$; i.e., $\exists \theta \in \hat{\Theta}_{t}$ such that 
$\mu_\arm(\theta) = \max_{\ell \in \setofArms} \mu_\ell(\theta)$.
\end{defn}

\begin{defn}[\empNonCompetitive arm]
\label{defn:non_competitive}
An arm $\arm$ is said to be \empNonCompetitive if it is {\em not} \empCompetitive; i.e., if  $\mu_\arm(\theta) < \max_{\ell \in \setofArms} \mu_\ell(\theta)$ for all $\theta \in \hat{\Theta}_{t}$.
\end{defn}

If an arm is \empNonCompetitive, then it cannot be optimal if the true parameter lies inside the confidence set $\hat{\Theta}_{\slot}$. These \empNonCompetitive arms are not considered in Step 3 of the algorithm for round $t+1$. However,  these arms can be \empCompetitive in subsequent rounds; see also \Cref{rem:successive}. 
For example, in \Cref{fig:stepbystep}(b), the mean reward of Arm 3 (shown in green) is strictly lower than the two other arms for all $\theta \in \hat{\Theta}_t$. Hence, this arm is declared as \empNonCompetitive and only Arms 1 and 2 are included in the competitive set $\competitiveArms_\slot$. In the rare case when $\hat{\Theta}_t$ is empty, we set $\competitiveArms_\slot = \{1, \dots, \numArms \}$ and go directly to step 3 below.

\begin{algorithm}[t]
\hrule 
\vspace{0.1in}
\begin{algorithmic}[1]
\STATE \textbf{Input:} Reward Functions $\{\mu_1, \mu_2 \ldots \mu_\numArms\}$  \\
\STATE \textbf{Initialize:} $\pulls_\arm = 0$ for all $\arm \in \{1, 2, \dots \numArms\}$
\FOR{ each round $\slot+1$}
\STATE \small{\textbf{Confidence set construction:}} \small{$$\hat{\Theta}_t  = \left\{ \theta: \forall{k} \in \mathcal{K}, \mid \mu_{k}(\theta) - \estimateMean_{k}(\slot) \mid < \sqrt{\frac{2 \alpha \sigma^2 \log \slot}{\pulls_{k}(t)}}\right\}. $$}
If $\hat{\Theta}_t$ is an empty set, then define $\competitiveArms_\slot = \{1, \dots, \numArms \}$ and go to step 6
\STATE \textbf{Define competitive set $\competitiveArms_\slot$:} 
{\small{$$\competitiveArms_\slot = \left\{\arm : \mu_\arm(\theta) = \max_{\ell \in \setofArms} \mu_\ell(\theta) \quad  \text{for some } \theta \in \hat{\Theta}_\slot\right\}.$$}}%
\STATE \textbf{UCB among competitive arms} 
{\small{$$\arm_{\slot+1} = \arg \max_{\arm \in \competitiveArms_\slot} \left( \hat{\mu}_\arm(\slot) + \sqrt{\frac{2 \alpha \sigma^2 \log \slot}{\pulls_\arm(\slot)}} \right).$$}}%
\STATE Update empirical mean $\hat{\mu}_\arm(t+1)$ and $\pulls_\arm(t+1)$ for arm $\arm_{t+1}$.
\ENDFOR
\end{algorithmic}
\vspace{0.1in}
\hrule
\caption{UCB-Competitive (UCB-C)}
\label{alg:formalAlgoUCBC}
\end{algorithm}

\begin{algorithm}[t]
\hrule 
\vspace{0.1in}
\begin{algorithmic}
\STATE1: Steps 1 to 5 as in \Cref{alg:formalAlgoUCBC}
\STATE2:  \small{\textbf{Apply Thompson sampling on $\competitiveArms_\slot$:} }
    \FOR  {$\arm \in \competitiveArms_\slot$}
    \STATE  Sample \small{$S_{\arm,\slot} \sim \mathcal{N}\left(\hat{\mu}_\arm(\slot), \frac{\beta  \sigma^2}{\pulls_\arm(\slot)}\right)$. }
    \ENDFOR
\STATE $\arm_{\slot+1} = \arg \max_{\arm \in \competitiveArms_\slot} S_{\arm,\slot}$
\STATE3: Update empirical mean, $\hat{\mu}_\arm$ and $\pulls_\arm$ for arm $\arm_{t+1}$.
\end{algorithmic}
\vspace{0.1in}
\hrule
\caption{Competitive Thompson Samp. (TS-C)}
\label{alg:formalAlgoTSC}
\end{algorithm}

\vspace{0.1cm} \noindent
\textbf{Step 3: Pull an arm from the set $\competitiveArms_t$ using a classical bandit algorithm.} At round $\slot+1$, we choose one of the \empCompetitive arms using any classical bandit \textsc{Algorithm} (for e.g., UCB, Thompson sampling, KL-UCB, or any algorithm to be developed for the classical bandit framework which does not explicitly model a structure connecting the rewards of different arms). Formal descriptions for UCB-C and TS-C, i.e., the structured bandit versions on UCB \citep{lai1985asymptotically, auer2002finite} and Thompson Sampling \citep{thompson1933likelihood} algorithms, are presented in \Cref{alg:formalAlgoUCBC} and \Cref{alg:formalAlgoTSC}, respectively. The ability to employ any bandit algorithm in its last step is an important advantage of our algorithm. In particular, Thompson sampling has attracted a lot of attention \citep{thompson1933likelihood, chapelle2011empirical, agrawal2012analysis, russo2017tutorial} due to its superior empirical performance. Extending it to the structured bandit setting results in significant regret improvement over previously proposed structured bandit algorithms.

\begin{rem}[Connection to successive elimination algorithms for best-arm identification]
\label{rem:successive}
Note that the empirically competitive set is updated at every round $t$. Thus, an arm that is empirically non-competitive at some round $\tau$ can be empirically competitive in subsequent rounds. Hence, the proposed algorithm is different from successive elimination methods used for best-arm identification \citep{bubeck2009pure, audibert2010best, jamieson2014best, jamieson2014lil}. 
Unlike successive elimination methods, the proposed algorithm does not permanently eliminate empirically non-competitive arms but allows them to become competitive again in subsequent rounds.
\end{rem}

\begin{rem}[Comparison with UCB-S proposed in \citep{lattimore2014bounded}] The paper \citep{lattimore2014bounded} proposes an algorithm called UCB-S for the same structured bandit framework considered in this work. UCB-S constructs the confidence set $\hat{\Theta}_t$ in the same way as Step 1 described above. It then pulls the arm $k = \arg \max_{k \in \setofArms} \sup_{\theta \in \hat{\Theta}_t} \mu_k(\theta)$. Taking the supremum of $\mu_k(\theta)$ over $\theta$ makes UCB-S sensitive to small changes in $\mu_k(\theta)$ and to the confidence set $\hat{\Theta}_t$. Our approach of identifying competitive arms is more robust, as observed in \Cref{sec:regret} and \Cref{sec:experiments}. Moreover, the flexibility of using Thompson Sampling in Step 3 results in a significant reduction in regret  over UCB-S. As noted in \citep{lattimore2014bounded}, the approach used to design UCB-S cannot be directly generalized to Thompson Sampling and other bandit algorithms.
\end{rem}

\begin{rem}[Computational complexity of \textsc{Algorithm}-C]
The computational complexity of \textsc{Algorithm}-C depends on the construction of $\hat{\Theta}_t$ and identifying $\hat{\Theta}_t$-competitive arms. The algorithm is easy to implement in cases where the set $\Theta$ is \textit{small} or in situations where the pre-image of mean reward functions $\mu_k(\theta)$ can be easily computed. For our simulations and experiments, we discretize the set $\Theta$ wherever $\Theta$ is uncountable.
\end{rem}

\section{Regret Analysis and Insights}
\label{sec:regret}

In this section, we evaluate the performance of the UCB-C algorithm through a finite-time analysis of the  expected cumulative regret defined as
\begin{equation}
\E{\regret(\totalPulls)} = \sum_{\arm = 1}^{\numArms} \E{\pulls_\arm(\totalPulls)} \gap_\arm,
\label{eq:regretEqn}
\end{equation}
where $\gap_\arm = \mu_{\bestarm}(\theta^*) - \mu_\arm(\theta^*)$ and $\pulls_\arm(\totalPulls)$ is the number of times arm $\arm$ is pulled in a total of  $\totalPulls$ time steps. To analyze the expected regret, we need to determine $\E{\pulls_\arm(\totalPulls)}$ for each sub-optimal arm $\arm \neq \bestarm$. We derive $\E{\pulls_\arm(\totalPulls)}$ separately for competitive and non-competitive arms. 
Our proof presents a novel technique to  show that each non-competitive arm is pulled only $\OO(1)$ times; i.e., our algorithms stop pulling non-competitive arms after some finite time. 
To establish the fact that competitive arms are pulled $\OO(\log T)$ times each, we prove that the proposed algorithm effectively reduces a $K$-armed bandit problem to a $C(\theta^*)$-armed bandit problem, allowing us to extend the regret analysis of the underlying classical multi-armed bandit algorithm (UCB, Thompson Sampling, etc.)

\subsection{Competitive and Non-competitive Arms}

In \Cref{sec:algorithmC}, we defined the notion of competitiveness of arms with respect to the confidence set $\hat{\Theta}_t$ at a fixed round $t$. For our regret analysis, we need asymptotic notions of competitiveness of arms, which are given below.

\begin{figure}[t!]
    \centering
    \includegraphics[width = 0.9\textwidth]{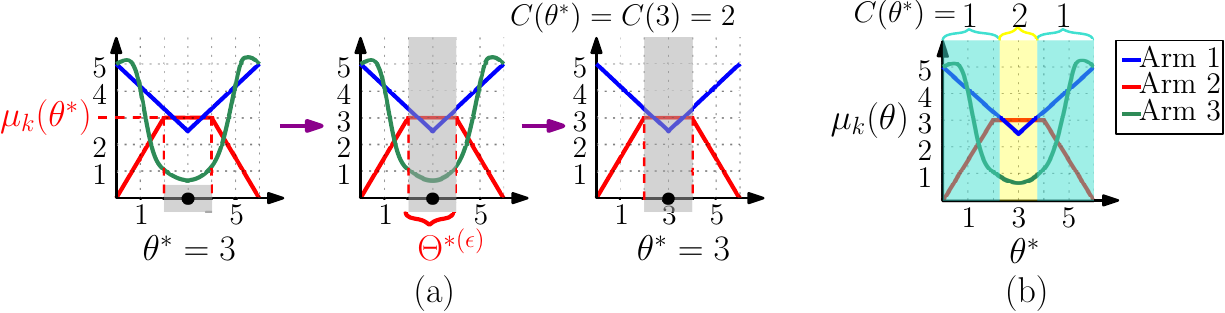}
    \caption{(a) Illustration of how the number of competitive arms $C(\theta^*)$ depends on the value of $\theta^*$ and the mean reward functions, when $\theta^* = 3$. To identify the competitive arms, we first find the set $\Theta^*(\epsilon) = \{\theta: |\mu_\bestarm(\theta^*) - \mu_\bestarm(\theta)| < \epsilon \}$ for small $\epsilon > 0$. Since Arm $3$ (shown in green) is sub-optimal for all $\theta \in \Theta^*(\epsilon)$ it is non-competitive. As a result, $C(\theta^*) = C(3) = 2$. (b) The number of competitive arms depend on the value of $\theta^*$. The grey region illustrates range of $\theta^*$ where $C(\theta^*) = 1$ and the yellow region indicates the range of values for which $C(\theta^*) = 2$.}
    \label{fig:C_theta_illustration}
\end{figure}

\begin{defn}[Non-competitive and Competitive Arms]
For any $\epsilon > 0$, let
\begin{align}
\Theta^{*(\epsilon)} = \{\theta: |\mu_\bestarm(\theta^*) - \mu_\bestarm(\theta)| < \epsilon \}. \nonumber
\end{align}
An arm $k$ is said to be non-competitive if there exists an $\epsilon > 0$ such that $k$ is not the optimal arm for any $\theta \in \Theta^{*(\epsilon)}$; i.e., 
if $\mu_\arm(\theta) < \max_{\ell \in \setofArms} \mu_\ell(\theta)$ for all $\theta \in \Theta^{*(\epsilon)}$.  
Otherwise, the arm is said to be competitive; i.e., if for all $\epsilon > 0$, $\exists \theta \in \Theta^{*(\epsilon)}$ such 
that
$\mu_\arm(\theta) = \max_{\ell \in \setofArms} \mu_\ell(\theta)$.
The number of competitive arms is denoted by $C(\theta^*)$. 
\end{defn}

Since the optimal arm $\bestarm$ is competitive by definition, we have
\[
1 \leq C(\theta^*) \leq K.
\]
We can think of $\Theta^{*(\epsilon)}$ as a confidence set for $\theta$ obtained from the samples of the best arm $\bestarm$. 
To intuitively understand the meaning of non-competitiveness, recall that the observed rewards $\hat{\mu}_k(t)$ from the arms help infer that $\theta^*$ lies in the confidence set $\hat{\Theta}_t$ with high probability. The observed reward $\hat{\mu}_{k^*}(t)$ of arm $k^*$ will dominate the construction of the confidence set $\hat{\Theta}_t$ because a good multi-armed bandit strategy pulls the optimal arm $\OO(t)$ times, while other arms are pulled at most $\OO(\log t)$ times. Thus, for any $\epsilon>0$, we expect the confidence set $\hat{\Theta}_t$ to converge to $\Theta^{*(\epsilon)}$
as the  number $n_{k^*}(t)$ of pulls for the optimal arm gets larger. As a result, if a sub-optimal arm $k$ is non-competitive as per the definition above,  i.e., $\mu_\arm(\theta) < \max_{\ell \in \setofArms} \mu_\ell(\theta)$ for all $\theta \in \Theta^{*(\epsilon)}$, then the proposed algorithm will identify  $k$ as $\hat{\Theta}_t$-non-competitive (and thus not pull it) with increasing probability at every round $t$. In fact, our regret analysis shows that the likelihood of a non-competitive arm being pulled at time $t$ decays as $t^{-1-\gamma}$ for some $\gamma>0$, leading to such arms being pulled only finitely many times.

We note that the number of competitive $C(\theta^*)$ arms is a function of the unknown parameter $\theta^*$ and the mean reward functions $\mu_k(\theta)$. \Cref{fig:C_theta_illustration}(a) illustrates how $C(\theta^*)$ is determined for the set of reward functions in \Cref{fig:problem_setup} and when $\theta^* = 3$. If $\theta^* = 3$, arm $2$ (shown in red) is optimal. The corresponding confidence set $\Theta^{*(\epsilon)}=[2-\frac{2\epsilon}{3},4+\frac{2\epsilon}{3}]$ is a slightly expanded version of the range of $\theta$ corresponding to the flat part of the reward function around $\theta^*$. Arm $3$ (shown in green) has sub-optimal mean reward $\mu_3(\theta)$ for all $\theta \in \Theta^{*(\epsilon)}$, and thus it is non-competitive. On the other hand, Arm $1$ (shown in blue) is competitive. Therefore, the number of competitive arms $C(\theta^*) = 2$ when $\theta^* =3$. \Cref{fig:C_theta_illustration}(b) shows how $C(\theta^*)$ changes with the value of $\theta^*$. When $\theta^*$ is outside of $[2,4]$, i.e., the  the flat portion of Arm 2, $\Theta^{*(\epsilon)}$ is a much smaller set and it is possible to show that both Arms $1$ and $3$ are non-competitive. Therefore, the number of competitive arms $C(\theta^*) = 1$ when $\theta^*$ is outside $[2,4]$.

\subsection{Upper Bounds on Regret}

\begin{defn}[Degree of Non-competitiveness, $\epsilon_k$]
The degree of non-competitiveness $\epsilon_k$ of a non-competitive arm $k$ is the largest $\epsilon$ for which  $\mu_\arm(\theta) < \max_{\ell \in \setofArms} \mu_\ell(\theta)$ for all $\theta \in \Theta^{*(\epsilon)}$, where $\Theta^{*(\epsilon)} = \{\theta: |\mu_\bestarm(\theta^*) - \mu_\bestarm(\theta)| < \epsilon \}$. In other words, $\epsilon_k$ is the largest $\epsilon$ for which arm $k$ is $\Theta^{*(\epsilon)}$-non-competitive.
\end{defn}

Our first result shows that the expected pulls for non-competitive arms are bounded with respect to time $T$. Arms with a larger degree of non-competitiveness $\epsilon_k$ are pulled fewer times. 
\begin{thm}[Expected pulls of each of the $K-C(\theta^*)$ non-competitive Arms]
If arm $k$ is non-competitive, then the number of times it is pulled by UCB-C is upper bounded as
\vspace{-1mm}
\begin{align}
    \E{\pulls_\arm(\totalPulls)} &\leq \numArms \slot_0 + \sum_{\slot = 1}^{\totalPulls} 2 \numArms \slot^{1-\alpha} + \numArms^3  \sum_{\slot = \numArms \slot_0}^{\totalPulls} 6 \left(\frac{\slot}{\numArms}\right)^{2-\alpha} \label{eqn:ts} \\
    &= \OO(1) \quad \text{for } \alpha > 3. \nonumber
\end{align}
\label{thm:nonCompUCB}

Here,
\vspace{-2mm}
{\small
\begin{align*}
&\slot_0  = \inf \bigg\{\tau \geq 2: \Delta_{\text{min}} ,\epsilon_k \geq 4 \sqrt{\frac{K\alpha\sigma^2\log \tau}{\tau}} \bigg\}; \quad \Delta_{\text{min}} = \min_{k \in \mathcal{K}}\Delta_k.
\end{align*}
\vspace{-3mm}
}

\label{thm:nonComp}
\end{thm}

The $\OO(1)$ constant depends on the degree of competitiveness $\epsilon_k$ through $t_0$. If $\epsilon_k$ is large, it means that $t_0$ is small and hence $\E{n_k(T)}$ is bounded above by a small constant. The second and third terms in \Cref{eqn:ts} sum up to a constant for $\alpha > 3, \beta > 1$.

The next result shows that expected pulls for any competitive arm is $\OO(\log T)$. This result holds for any sub-optimal arm, but for non-competitive arms we have a stronger upper bound (of $\OO(1)$) as given in \Cref{thm:nonComp}. Regret analysis of UCB-C is presented in Appendix E. In Appendix D, we present a unified technique to prove results for any other \textsc{Algorithm-C}, going beyond UCB-C. We present the regret analysis of TS-C (with Beta prior and $K=2$) in Appendix F.

\begin{thm}[Expected pulls for each of the $C(\theta^*)-1$ competitive sub-optimal arms]
The expected number of times a competitive sub-optimal arm is pulled by UCB-C Algorithm is upper bounded as 
\begin{align*}
    \vspace{-1mm}
     \E{n_k(T)} &\leq 8\alpha \sigma^2 \frac{\log T}{\Delta_k^2} + \frac{2\alpha}{\alpha-2} + \sum_{\slot = 1}^{\totalPulls} 2\numArms \slot^{1-\alpha} \nonumber \\
     &= \OO(\log T) \quad \text{for } \alpha > 2,
\end{align*}
\label{thm:CompUCB}

\end{thm}

Plugging the results of \Cref{thm:nonComp} and \Cref{thm:CompUCB} in \eqref{eq:regretEqn} yields the bound on the expected regret in \Cref{thm:regret_bnd}. Note that in this work we consider a finite-armed setting where the number of arms $K$ is a fixed constant that does {\em not} scale with $T$ -- we focus on understanding how the cumulative regret scales with $T$ while $K$ remains constant.

\subsection{Proof Sketch}
We now present the proof sketch for \Cref{thm:nonComp}. The detail proof is given in the Appendix. For UCB-C, the proof can be divided into three steps presented below. The analysis is unique to our paper and allows us to prove that the UCB-C algorithm pulls the non-competitive arms only $\OO(1)$ times. The key strength of our approach is that the analysis can be easily extended to any \textsc{Algorithm}-C. 

\vspace{0.1cm}
\noindent \textbf{i) The probability of arm $k^*$ being $\hat{\Theta}_t$-Non-Competitive is small.} Observe that $\theta^* \in \hat{\Theta}_t$ implies that $k^*$ is $\hat{\Theta}_t$-competitive. Let $E_1(t)$ denote the event that the optimal arm $k^*$ is $\hat{\Theta}_t$-non-competitive at round $t$. As we obtain more and more samples, the probability of $\theta^*$ lying outside $\hat{\Theta}_t$ decreases with $t$. Using this, we show that (viz.~ Lemma 3 in the Appendix)
\begin{equation}
    \Pr(E_1(t))  
\leq 2Kt^{1-\alpha}.
\label{eq:Lemma3_result}
\end{equation}
This enable us to bound the expected number of pulls of a competitive arm as follows.
\begin{equation}
    \E{n_k(t)} \leq \sum_{t = 1}^{T} \Pr(E_1(t)) + \sum_{t = 0}^{T-1} \Pr(I_k(t) > I_{k^*}(t), k_{t+1} = k).
\label{eqn:lookatfirstterm}
\end{equation}
In view of \Cref{eq:Lemma3_result}, the first term in \Cref{eqn:lookatfirstterm} sums up to a constant for $\alpha > 2$. The term $I_k(t)$ represents the UCB Index if the last step in the algorithm is UCB, i.e., $I_k(t) = \hat{\mu}_k(t) + \sqrt{\frac{2\alpha \sigma^2 \log t}{n_k(t)}}$. The analysis of second term is exactly same as that for the UCB algorithm \citep{auer2002finite}. Due to this, the upper bound on expected number of pulls of competitive arms using UCB-C has the same pre-log constants as UCB. 

\noindent
\textbf{ii) The probability of a non-competitive arm being pulled jointly with the event that $n_{k^*}(t) > t/K$ is small.} Consider the joint event that a non-competitive arm with parameter $\epsilon_k$ is pulled at round $t+1$ and the number of pulls of optimal arm till round $t$ is at least $t/K$. In Lemma 4 in Appendix C we show that this event is unlikely. Intuitively, this is because when arm $k^*$ is pulled sufficiently many times, the confidence interval of mean of optimal arm is unlikely to contain any value outside $[\mu_{k^*}(\theta^*) - \epsilon_k, \mu_{k^*}(\theta^*) + \epsilon_k]$. 
Due to this, with high probability, arm $k$ is eliminated for round $t+1$ in step 2 of the algorithm itself. This leads to the result of Lemma 4 in the appendix, 
\begin{equation}
    \Pr\left(\arm_{\slot + 1} = \arm , n_{k^*}(t) > \frac{t}{K}\right) \leq 2\slot^{1-\alpha} \quad \forall{t > t_0}
\label{eqn:step2see}
\end{equation}

\vspace{0.1cm}
\noindent
\textbf{iii) The probability that a sub-optimal arm is pulled more than $\frac{t}{K}$ times till round $t$ is small.} In Lemma 6 in the Appendix, we show that 
\begin{equation}
    \Pr\left(n_k(t) > \frac{t}{K}\right) \leq 6K^2\left(\frac{t}{K}\right)^{2-\alpha} \quad \forall{t > Kt_0}.
\label{eqn:step3see}
\end{equation}
This result is specific to the last step used in \textsc{Algorithm}-C. To show \Cref{eqn:step3see} we first derive an intermediate result for UCB-C which states that $$\Pr(\arm_{\slot + 1} = \arm , \pulls_\arm(\slot) \geq s) \leq (2\numArms + 4) \slot^{1-\alpha} \quad \text{for }  s \geq \frac{\slot}{2 \numArms}.$$ Intuitively, if we have large number of samples of arm $k$, its UCB index is likely to be close to $\mu_k$, which is unlikely to be larger than the UCB index of optimal arm $k^*$ (which is around $\mu_{k^*}$ if $n_{k^*}$ is also large, or even higher if $n_{k^*}$ is \emph{small} due to the exploration term added in UCB index).  

The analysis of steps ii) and iii) are unique to our paper and help us obtain the $\OO(1)$ regret for non-competitive arms. Using these results, we can write the expected number of pulls for a non-competitive arm as
\begin{align}
\E{n_k(t)} &\leq Kt_0 + \sum_{t = Kt_0}^{T-1} \Pr\left(\arm_{\slot + 1} = \arm , n_{k^*}(t) = \max_{k \in \mathcal{K}} n_k(t)\right)  \nonumber \\
&+ \sum_{t = Kt_0}^{T-1} \sum_{k \in \mathcal{K}, k \neq k^*} \Pr(n_k(t) = \max_{k \in \mathcal{K}} n_k(t)).
\label{eqn:mainResultproof}
\end{align}
The second term in \Cref{eqn:mainResultproof} is bounded through step ii) (viz.~\Cref{eqn:step2see}) and the third term in \Cref{eqn:mainResultproof} is bounded for each sub-optimal arm through step iii) (viz.~\Cref{eqn:step3see}). Together, steps ii) and iii) imply that the expected number of pulls for a non-competitive arm is bounded.

\subsection{Discussion on Regret Bounds}
\label{subsec:sim}

\vspace{0.1cm}
\noindent
\textbf{Reduction in the effective number of arms.} The classical multi-armed bandit algorithms, which are agnostic to the structure of the problem, pull each of the $(K-1)$ sub-optimal arms $\OO(\log T)$ times. In contrast, our UCB-C algorithm pulls only a {\em subset} of the sub-optimal arms $\OO(\log T)$ times, with the rest (i.e., non-competitive arms) being pulled only $\OO(1)$ times. More precisely, our algorithms pull each of the $C(\theta^*)-1 \leq K-1$ arms that are competitive but sub-optimal $\OO(\log T)$ times. It is important to note that the upper bound on the pulls of these competitive arms by UCB-C has the same pre-log constants with that of the UCB, as shown in \Cref{thm:regret_bnd}. Consequently, the ability of UCB-C to reduce the pulls of non-competitive arms from $\OO(\log T)$ to $\OO(1)$ results directly in it achieving a smaller cumulative regret than its non-structured counterpart.

The number of competitive arms, i.e., $C(\theta^*)$, depends on the functions $\mu_1(\theta), \ldots, \mu_K(\theta)$ {\em as well as} the hidden parameter $\theta^*$. Depending on $\theta^*$, it is possible to have $C(\theta^*) = 1$, or $C(\theta^*)=K$, or any number in between. When $C(\theta^*) = 1$, all sub-optimal arms are non-competitive due to which our proposed algorithms achieve $\OO(1)$ regret. What makes our algorithms appealing is the fact that even though they do not explicitly try to predict the set (or, the number) of competitive arms, they {\em stop} pulling any non-competitive arm after finitely many steps.

\begin{figure}[t]
         \centering
         \includegraphics[width = 0.6\textwidth]{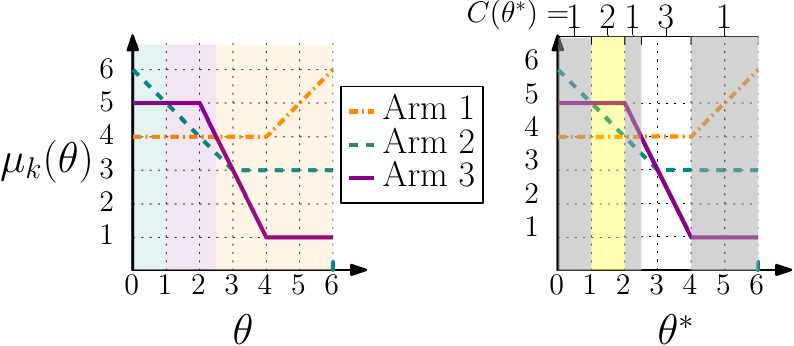}
         \caption{\sl (left) Arm 2 is optimal for $\theta^* \in [0,1]$, Arm 3 is optimal for $\theta^* \in [1,2.5]$ and Arm 1 is optimal for $\theta^* \in [2.5,6]$, (right) the number of competitive arms for different ranges of $\theta$ shaded in grey ($C(\theta) = 1$), yellow ($C(\theta) = 2)$ and white ($C(\theta) = 3$).}
         \label{fig:simEx1}
\end{figure}

\begin{figure}[t]
    \centering
    \includegraphics[width = 0.6\textwidth]{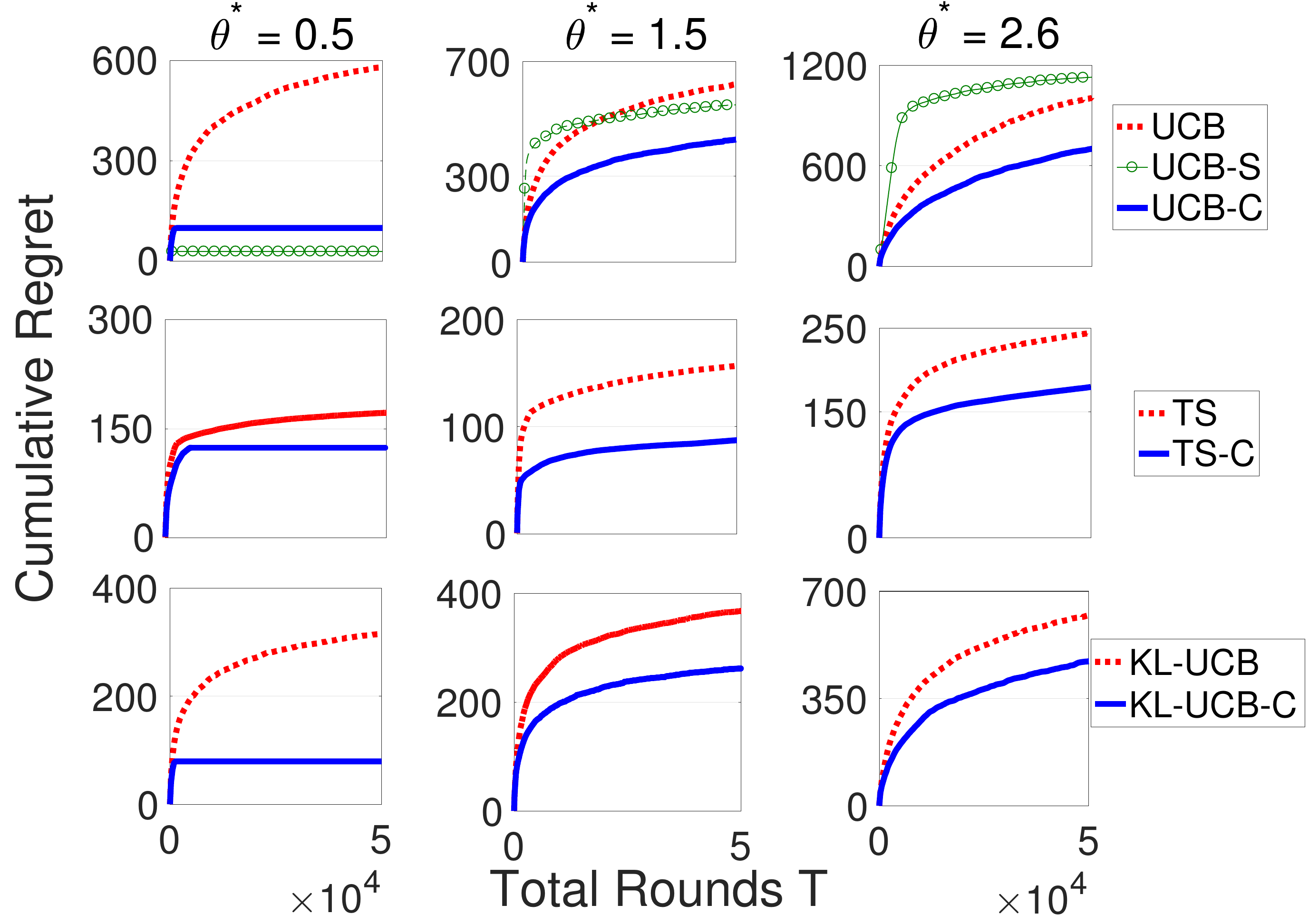}
    \caption{\sl Cumulative regret of \textsc{Algorithm-C} vs.~\textsc{Algorithm} (UCB in row 1, TS in row 2 and KL-UCB in row 3) for the setting in \Cref{fig:simEx1}. The number of competitive arms is $C(\theta^*) = 1$ in the first column, $C(\theta^*) = 2$ in second column and $C(\theta^*) = 3$ in third column. Unlike UCB-S which only extends UCB, our approach generalizes any classical bandit algorithm such as UCB, TS, and KL-UCB to the structured bandit setting.}
    \label{fig:Simulation1}
\end{figure}

\vspace{0.1cm}
\noindent
\textbf{Empirical performance of \textsc{Algorithm-C}.} In \Cref{fig:Simulation1} we compare the regret of \textsc{Algorithm-C} against the regret of \textsc{Algorithm} (UCB/TS/KL-UCB). We plot the cumulative regret attained under \textsc{Algorithm-C} vs.~\textsc{Algorithm} of the example shown in \Cref{fig:simEx1} for three different values of $\theta^*: 0.5,1.5$ and $2.6$. Refer to \Cref{fig:simEx1} to see that $C = 1$, $2$ and $3$ for $\theta^* = 0.5$,$1.5$ and $2.6$, respectively. Due to this, we see that \textsc{Algorithm-C} achieves bounded regret for $\theta^* = 0.5$, and reduced regret relative to \textsc{Algorithm} for $\theta^* = 1.5$ as only one arm is pulled $\OO(\log T)$ times. For $\theta^* = 2.6$, even though 
 $C = 3$ (i.e., all arms are competitive), \textsc{Algorithm-C} achieves empirically smaller regret than \textsc{Algorithm}. We  also see the advantage of using TS-C and KL-UCB-C over UCB-C in \Cref{fig:Simulation1} as Thompson Sampling and KL-UCB are known to outperform UCB empirically. For all the simulations, we set $\alpha = 3, \beta = 1$. Rewards are drawn from the distribution $\mathcal{N}(\mu_{k}(\theta^*),4)$, i.e, $\sigma = 2$. We average the regret over $100$ experiments. For a given experiment, all algorithms use the same reward realizations.

\begin{figure}[t]
    \centering
    \includegraphics[width = 0.57\textwidth]{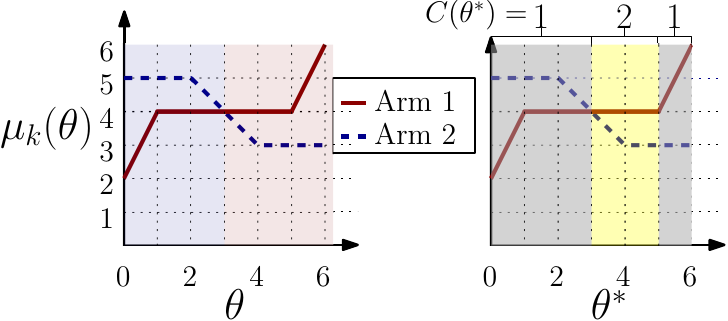}
    \caption{\sl Arm 2 is optimal for $\theta^* \in [0,3]$ and Arm 1 is optimal for $\theta^* \in [3,5]$. For $\theta \in [0,3] \cup [5,6]$, $C(\theta) = 1$ and $C(\theta) = 2$ for $\theta \in [3,5]$.}
    \label{fig:exSim2}
\end{figure}

\begin{figure}[t]
    \centering
    \includegraphics[width = 0.57\textwidth]{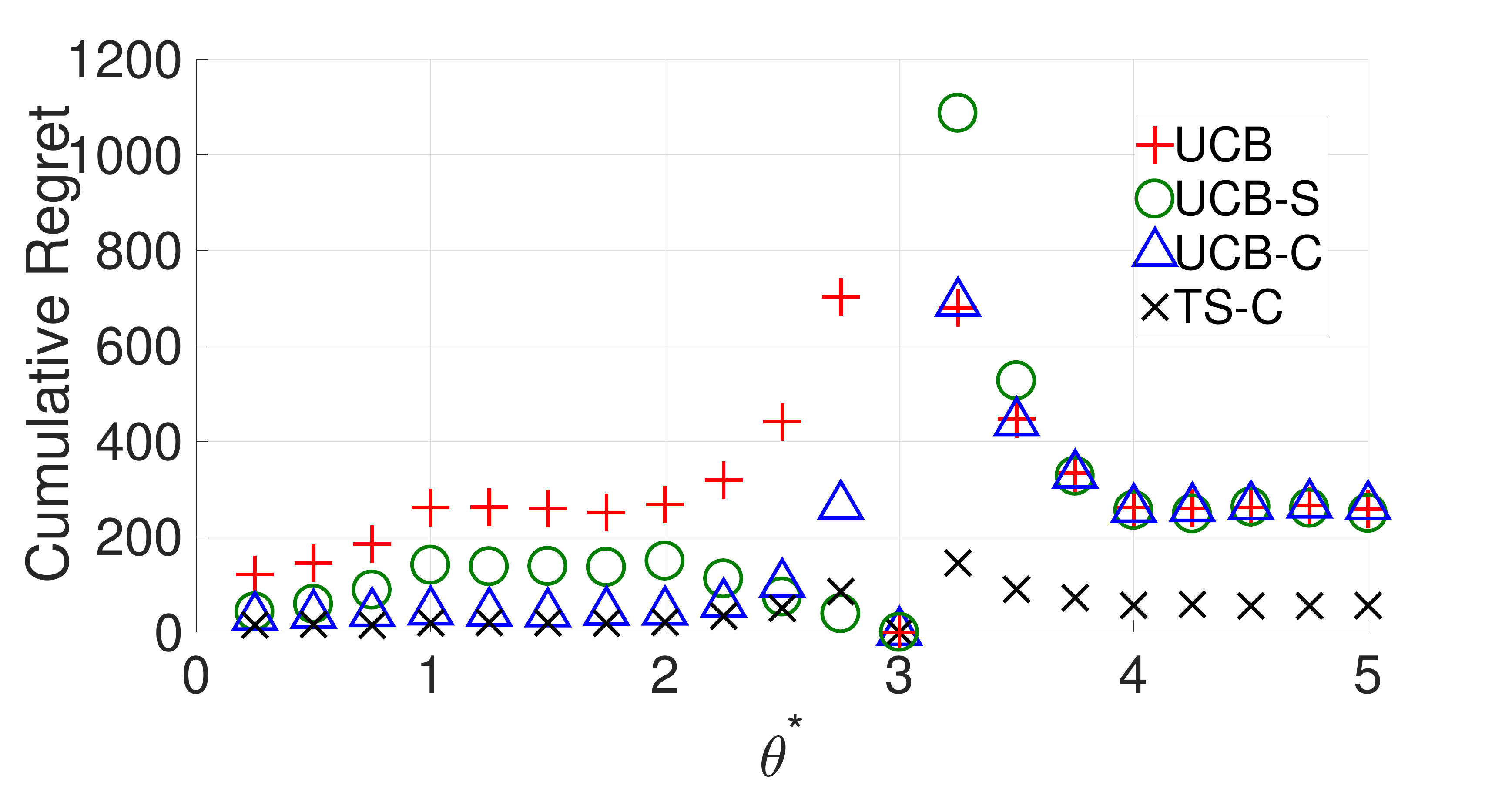}
    \caption{\sl Cumulative regret of UCB,UCB-S,UCB-C and TS-C versus $\theta^*$ for the example in \Cref{fig:exSim2} over $50000$ runs. UCB-S is sensitive to the value of $\theta^*$ and the reward functions as it is seen to achieve a small regret for $\theta^* = 2.75$, but obtains a worse regret than UCB for $\theta^* = 3.25$.}
    \label{fig:Sim2}
\end{figure}

\vspace{0.1cm}
\noindent
\textbf{Performance comparison with UCB-S.} In the first row of \Cref{fig:Simulation1}, we also plot the performance of the UCB-S algorithm proposed in \citep{lattimore2014bounded}, alongside UCB and UCB-C. The UCB-S algorithm constructs the confidence set $\hat{\Theta}_t$ just like UCB-C, and then in the next step selects the arm $k_{t+1} = \arg \max_{k \in \mathcal{K}} \sup_{\theta \in \hat{\Theta}_t} \mu_k(\theta)$. Informally, it finds the maximum possible mean reward $\mu_k(\theta)$ over  $\theta \in \hat{\Theta}$ for each arm $k$. As a result, UCB-S tends to favor pulling arms that have the largest mean reward for $\theta \in \Theta^{*(\epsilon)}$. This bias renders the performance of UCB-S to depend heavily on $\theta^*$. When $\theta^* = 0.5$,  UCB-S  has the smallest regret among the three algorithms compared in \Cref{fig:Simulation1}, but when $\theta^* = 2.6$  it gives even worse regret than UCB. A similar observation can be made in another simulation setting described below.

\Cref{fig:Sim2} compares UCB, UCB-S, UCB-C and TS-C for the functions shown in \Cref{fig:exSim2}. We plot the cumulative regret after 50000 rounds for different values of $\theta^* \in [0,5]$ and observe that TS-C performs the best for most $\theta^*$ values. As before, the  performance of UCB-S varies significantly with $\theta^*$. In particular, UCB-S has the smallest regret of all when $\theta^* = 2.75$, but achieves worse regret  even compared to UCB when $\theta^* = 3.25$. On the other hand, our UCB-C performs better than or at least as good as UCB for all   $\theta^*$. While UCB-S also achieves the regret bound of \Cref{thm:regret_bnd}, the ability to employ any \textsc{Algorithm} in the last step of \textsc{Algorithm-C} is a key advantage over UCB-S, as Thompson Sampling and KL-UCB can have significantly better empirical performance over UCB.

\begin{figure}[t]
    \centering
    \includegraphics[width = 0.5\textwidth]{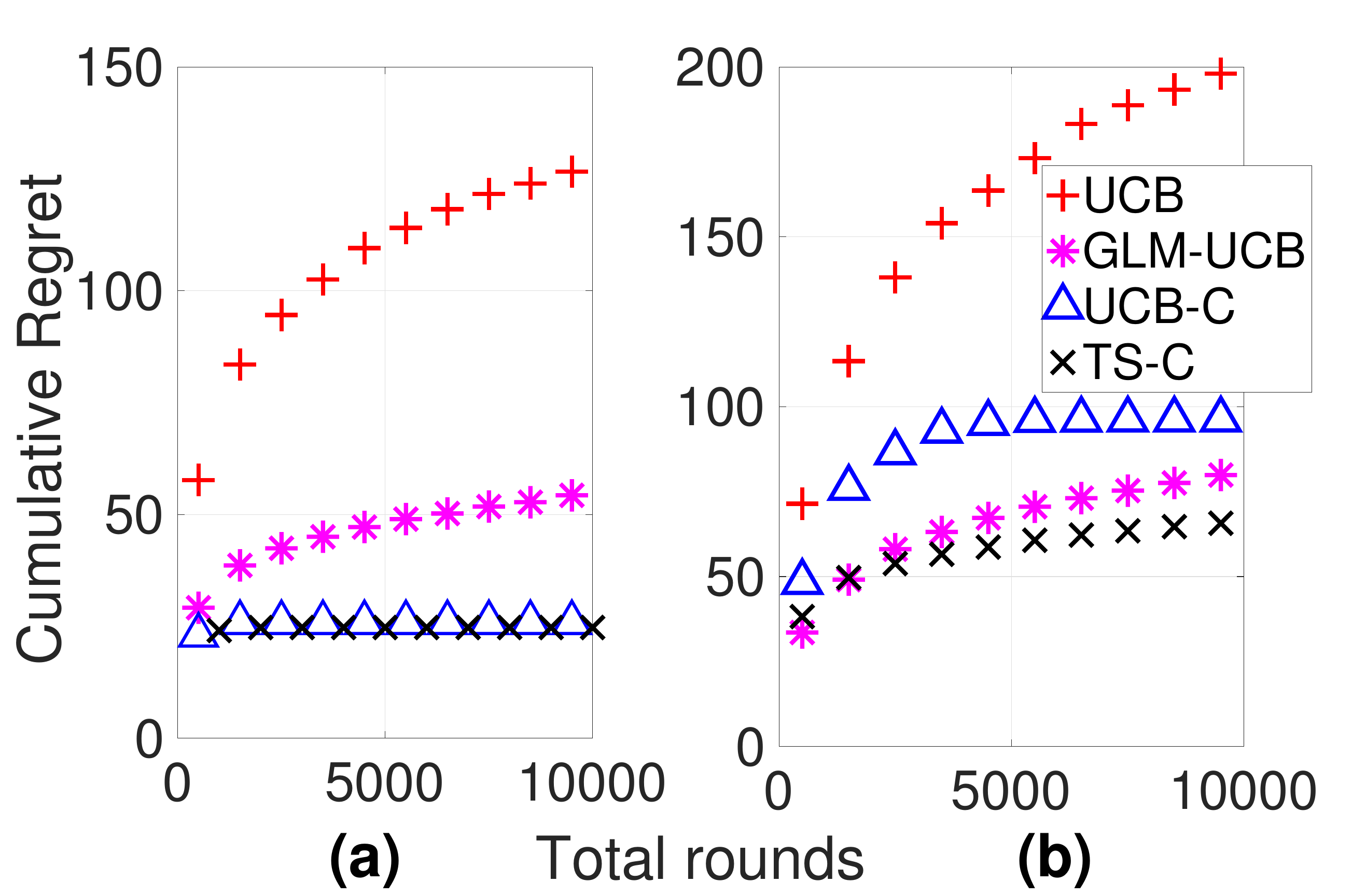}
    \caption{Cumulative regret of UCB, GLM-UCB, UCB-C and TS-C in the linear bandit setting, with $x_1 = (2,1), x_2 = (1,1.5)$ and $x_3 = (3,-1)$. Mean rewards are $(\vec{\theta}^*)^\intercal x_k$, with $\theta^* = (0.9, 0.9)$ in (a) and $\theta^* = (0.5, 0.5)$ in (b). While UCB-C and TS-C are designed for a much broader class of problems, they show competitive performance relative to GLM-UCB, which is a specialized algorithm for the linear bandit setting.}
    \label{fig:linBandit}
\end{figure}

\begin{figure}[t]
    \centering
    \includegraphics[width = 0.55\textwidth]{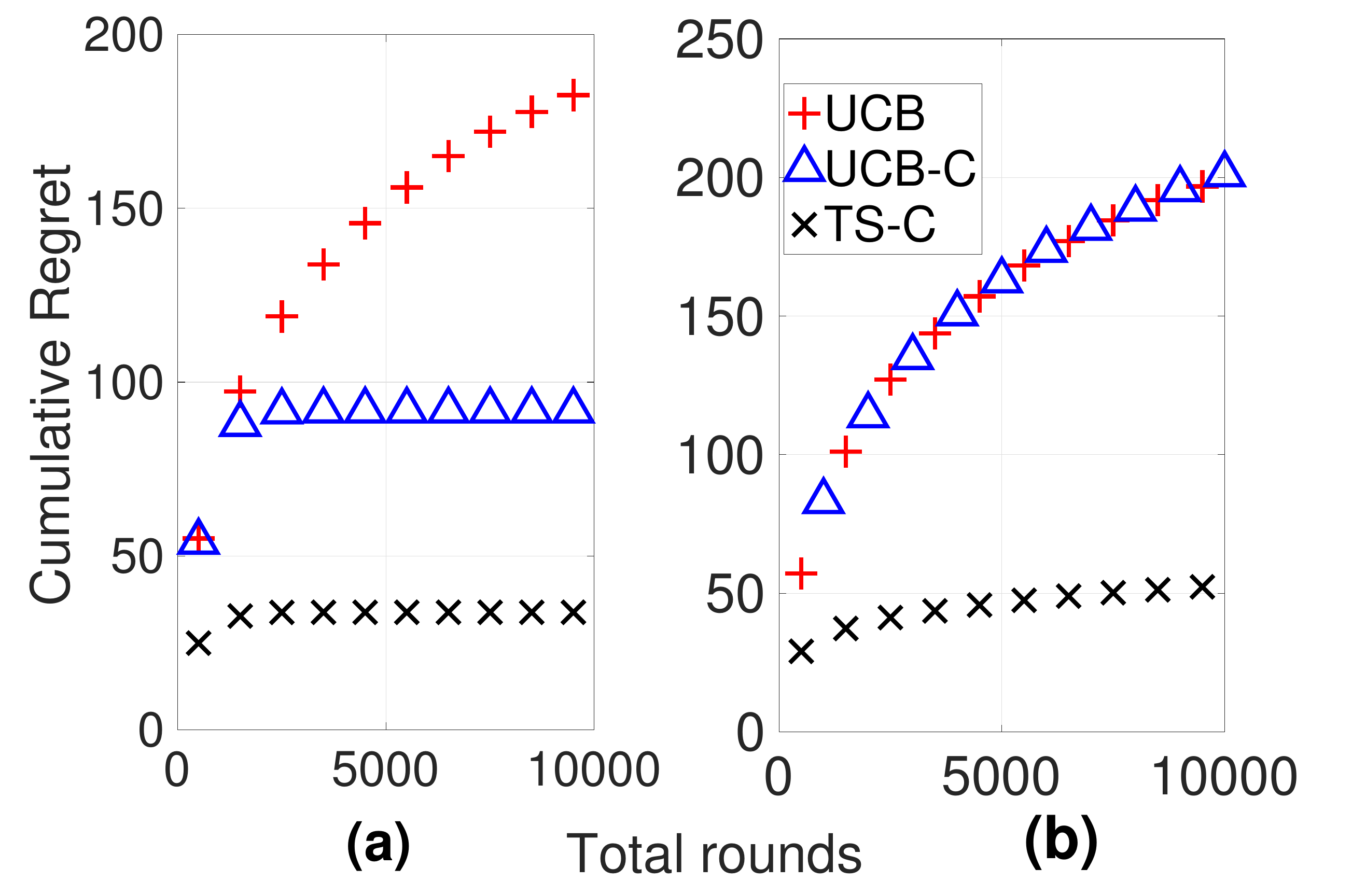}
    \caption{Cumulative regret for UCB,UCB-C and TS-C for the case in which $\theta \in [-1,1] \times [-1,1]$. The reward functions are $\mu_1(\vec{\theta}) = \theta_1 + \theta_2, \mu_2(\vec{\theta}) = \theta_1 - \theta_2$, and $\mu_3(\vec{\theta}) = \max(|\theta_1|,|\theta_2|)$. The true parameter $\vec{\theta}^*$ is $(0.9,0.2)$ in (a) and $(-0.2,0.1)$ in (b). The value of $C(\theta^*)$ is $1,3$ in (a) and (b) respectively.}
    \label{fig:multiTheta}
\end{figure}

\vspace{0.1cm} \noindent
\textbf{Comparison in linear bandit and multi-dimensional $\theta$ settings.} As highlighted in \Cref{sec:model}, our problem formulation allows $\theta$ to be multi-dimensional as well. \Cref{fig:linBandit} shows the performance of UCB-C and TS-C relative to GLM-UCB in a linear bandit setting. In a linear bandit setting, mean reward of arm $k$ is $\mu_k(\theta^*) = (\theta^*)^{\intercal}x_k $. Here $x_k$ is a vector associated with arm $k$, which is known to the player. The parameter $\theta^*$ is unknown to the player, and hence it fits in our structured bandit framework. It is important to see that while UCB-C and TS-C are designed for a much broader class of problems, they still show competitive performance relative to specialized algorithms (i.e., GLM-UCB) in the linear bandit setting (\Cref{fig:linBandit}). \Cref{fig:multiTheta} shows a setting in which $\theta$ is multi-dimensional, but the reward mappings are non-linear and hence the setting is not captured through a linear bandit framework. Our results in \Cref{fig:multiTheta} demonstrate that the UCB-C and TS-C algorithms work in such settings as well while providing significant improvements over UCB in certain cases.

\subsection{When do we get bounded regret?}
 When $C(\theta^*)=1$, all sub-optimal arms are pulled only $\OO(1)$ times, leading to a bounded regret. Cases with $C(\theta^*) = 1$ can arise quite often in practical settings. For example, when functions are continuous or $\Theta$ is countable, this occurs when the optimal arm $\bestarm$ is {\em invertible}, or has a unique maximum at $\mu_{k^*}(\theta^*)$, or any case where
the set  $\Theta^* = \{\theta: \mu_{\arm^*}(\theta) = \mu_{\arm^*}(\theta^*)  \}$ is a {\em singleton}.
These cases lead to  all sub-optimal arms  being non-competitive, whence UCB-C achieves bounded (i.e., $O(1)$) regret. There are more general scenarios where bounded regret is possible. To formally present such cases, we utilize a lower bound obtained in \citep{combes2017minimal}. 

\begin{prop}[Lower bound]
For any uniformly good algorithm \citep{lai1985asymptotically}, and for any $\theta \in \Theta$, we have: 
$$\lim \inf_{\totalPulls \rightarrow \infty} \frac{Reg(\totalPulls)}{\log \totalPulls} \geq L(\theta), \text{ where}$$
$$L(\theta) = 
\begin{cases}
0 \quad &\text{if } \tilde{C}(\theta^*) = 1, \\
> 0 \quad &\text{if } \tilde{C}(\theta^*) > 1.
\end{cases}
$$
An algorithm $\pi$ is uniformly good if $Reg^\pi(\totalPulls,\theta) = \oo(T^a)$ for all $a > 0$ and all $\theta \in \Theta$. Here $\tilde{C}(\theta^*)$ is the number of arms that are $\Theta^*$-Competitive, with $\Theta^*$ being the set $\{\theta : \mu_{k^*}(\theta) = \mu_{k^*}(\theta^*)\}$.
\label{prop:lowerBound}
\end{prop}
This suggests that bounded regret is possible only when $\tilde{C}(\theta) = 1$ and logarithmic regret is unavoidable in all other cases. The proof of this proposition follows from a bound derived in \citep{combes2017minimal} and it is given in Appendix B. 

There is a subtle difference between $C(\theta^*)$ and $\tilde{C}(\theta^*)$. This arises in corner case situations when a $\Theta^*$-Non-Competitive arm is competitive. Note that the set $\Theta^* = \{\theta:  \mu_{k^*}(\theta^*) = \mu_{k^*}(\theta) \}$ can be interpreted as the confidence set obtained when we pull the optimal arm $k^*$ infinitely many times. In practice, if we sample the optimal arm a \emph{large} number of times, we can only obtain the confidence set $\Theta^{*(\epsilon)} = \{\theta: |\mu_{k^*}(\theta^*) - \mu_{k^*}(\theta)| < \epsilon \}$ for some $\epsilon > 0$. Due to this, there is a difference between $\tilde{C}(\theta^*)$ and $C(\theta^*)$. Consider the case shown in \Cref{fig:cornerCompetitive} with $\theta^* = 3$. For $\theta^* = 3$, Arm 1 is optimal. In this case $\Theta^* = [2,4]$. For all values of $\theta \in \Theta^*$, $\mu_2(\theta) \leq \mu_1(\theta)$ and hence Arm 2 is $\Theta^*$-Non-Competitive. However, for any $\epsilon > 0$, Arm 2 is $\Theta^{*(\epsilon)}$-competitive and hence Competitive. Due to this, we have $\tilde{C}(3) = 1$ and $C(3) = 2$ in this case. 

If $\Theta$ is a countable set, a $\Theta^*$-Non-Competitive arm is always $\Theta^{*(\epsilon)}$-Non-Competitive, that is, $\tilde{C}(\theta^*) = C(\theta^*)$. This occurs because one can always choose $\epsilon = \min_{\{\theta \in \Theta \setminus \Theta^*\}} \{|\mu_{k^*}(\theta^*) - \mu_{k^*}(\theta)|\}$ so that a $\Theta^*$-Non-Competitive arm is also $\Theta^{*(\epsilon)}$-Non-Competitive. This shows that when $\Theta$ is a countable set (which is true for most practical situations where the hidden parameter $\theta$ is {\em discrete}), UCB-C achieves bounded regret \textit{whenever possible}, that is, whenever $\tilde{C}(\theta^*)=1$. While this property holds true for the case when $\Theta$ is a countable set, there can be more general cases where $C(\theta) = \tilde{C}(\theta)$. Our algorithms and regret analysis are valid regardless of $\Theta$ being countable or not.

\begin{figure}[t]
    \centering
    \includegraphics[width = 0.46\textwidth]{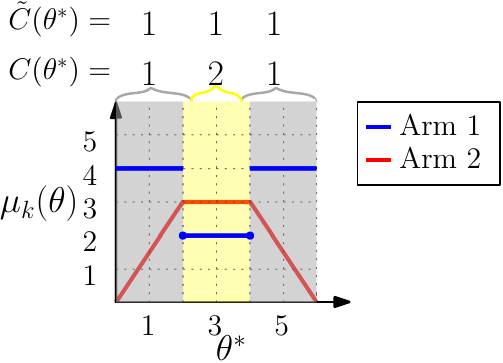}
    \caption{\sl For values of $\theta \in [2,4]$ Arm 2 is $\Theta^*$-Non-Competitive but it is still Competitive. As for any set slightly bigger than $\Theta$, i.e., $\Theta^{*(\epsilon)}$, it is $\Theta^{*(\epsilon)}$-Competitive. Hence this is one of the corner case situations where $C(\theta)$ and $\tilde{C}(\theta^*)$ are different.}
    \label{fig:cornerCompetitive}
    \vspace{-4mm}
\end{figure}

\section{Additional Exploration of Non-competitive But  Informative Arms}
The previous discussion shows that the UCB-C and TS-C algorithms enable substantial reductions in the effective number of arms and the expected cumulative regret. 
A strength of the proposed algorithms that can be a weakness in some cases is that they stop pulling non-competitive arms that are unlikely to be optimal after some finite number of steps. Although an arm may be non-competitive in terms of its reward yield, it can be useful in inferring the hidden parameter $\theta^*$, which in turn may help reduce the regret incurred in subsequent steps. For instance, consider the example shown in \Cref{fig:simyo2}. Here, Arm 3 is sub-optimal for all values of $\theta^* \in [0,6]$ and is never pulled by UCB-C, but it can help identify whether $\theta^* \geq 3$ or $\theta^* < 3$.
Motivated by this, we propose an add-on to Algorithm-C, named as the Informative Algorithm-C (\Cref{alg:Informative}), that takes the \emph{informativeness} of arms into account and performs additional exploration of the \textit{most informative arm} with a probability that decreases over time.

\begin{figure}[ht]
    \centering
    \includegraphics[width = 0.43\textwidth]{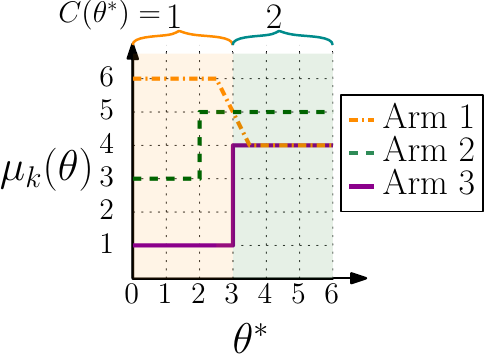}
    \caption{\sl In this example, Arm 3 has $\mu_3(\theta^*) = 1$ for $\theta^* < 3$ and $\mu_3(\theta^*) = 4$ for $\theta^* \geq 3$. See that Arm 3 is sub-optimal for all values of $\theta^*$, and hence is non-competitive for all $\theta^*$. However, a few pulls of Arm 3 can still be useful in getting some information on whether $\theta^* \geq 3$ or $\theta^* < 3$. 
    }
    \label{fig:simyo2}
\end{figure}

\subsection{Informativeness of an Arm}
Intuitively, an arm is informative if it helps us to obtain information about the hidden parameter $\theta^*$. At the end of round $t$, we know a confidence interval $\hat{\Theta}_t$ for the hidden parameter $\theta^*$. We aim to quantify the informativeness of an arm with respect to this confidence set $\hat{\Theta}_t$. For instance, if $\hat{\Theta}_t \in [2,4]$ in \Cref{fig:simyo2}, we see that the reward function of Arm 3 $\mu_3(\theta)$ has high variance and it suggests that the samples of Arm 3 could be helpful in knowing about $\theta^*$. On the other hand, samples of Arm 2 will not be useful in identifying $\theta^*$ if $\hat{\Theta}_t = [2,4]$. There can be several ways of defining the informativeness $I_k(\hat{\Theta}_t)$ of an arm with respect to set $\hat{\Theta}_t$. We consider the following two metrics in this paper.

\vspace{0.2cm} \noindent  \textbf{KL-Divergence.} Assuming that $\theta$ has a uniform distribution in $\hat{\Theta}_t$, we can define the informativeness $I_k(\hat{\Theta}_t)$ of an arm as the expected KL-Divergence between two samples of arm $k$, i.e., $I_k(\hat{\Theta}_t) = \mathbb{E}_{\theta_1, \theta_2} \left[ D_{KL}(f_{R_k}(R_k|\theta_1), f_{R_k}(R_k| \theta_2)) \right]$. Our intuition here is that larger expected KL-divergence for an arm indicates that samples from it have substantially different distributions under different $\theta^*$ values, which in turn indicates that those samples will be useful in inferring the true value of $\theta^*$.
Assuming that $\Pr(R_k|\theta)$ is a Gaussian distribution with mean $\mu_k(\theta)$ and variance $\sigma^2$, then the expected KL-Divergence can be simplified as
    \begin{align*}
        &\mathbb{E}_{\theta_1, \theta_2} \left[ D_{KL}(f_{R_k}(R_k|\theta_1), f_{R_k}(R_k| \theta_2)) \right]  \nonumber \\
        &= \mathbb{E}_{\theta_1, \theta_2} \left[ D_{KL}(\mathcal{N}(\mu_k(\theta_1), \sigma^2), \mathcal{N}(\mu_k(\theta_2), \sigma^2)) \right] \\
        &= \mathbb{E}_{\theta_1, \theta_2} \left[\frac{1}{2}(\mu_k(\theta_1) - \mu_k(\theta_2))^2 \right] \\
        &= \int_{\hat{\Theta}_t} \left(\mu_k(\theta) - \int_{\hat{\Theta}_t}\mu_k(\theta) U(\theta) \textrm{d}\theta \right)^2 U(\theta)\textrm{d}\theta = V_k(\hat{\Theta}_t), 
    \end{align*}
    where, $V_k(\hat{\Theta}_t)$ is the variance in the mean reward function $\mu_k(\theta)$, calculated when $\theta$ is uniformly distributed over the current confidence set $\hat{\Theta}_t$. Observe that the metric $I_k(\hat{\Theta}_t) = V_k(\hat{\Theta}_t)$ is easy to evaluate given the functions $\mu_k(\theta)$ and the confidence set obtained from Step 1.
    
\vspace{0.2cm} \noindent \textbf{Entropy.} Alternatively, $\mu_k(\theta)$ can be viewed as a derived random variable of $\theta$, where $\theta$ is uniformly distributed over the current confidence set $\hat{\Theta}_t$. The informativeness of arm $k$ can then be defined as $I_k(\hat{\Theta}_t) = H(\mu_k(\theta))$. When $\mu_k(\theta)$ is discrete this will be the Shannon entropy $H(\mu_k(\theta)) = \sum_{\theta \in \hat{\Theta}_t} -\Pr(\mu_k(\theta)) \log(\Pr(\mu_k(\theta)))$, while for continuous $\mu_k(\theta)$ it will be the {\em differential entropy} $H(\mu_k(\theta)) = \int_{\hat{\Theta}_t} -f_{\mu_k(\theta)} \log(f_{\mu_k(\theta)}) d(\mu_k(\theta))$ where $f_{\mu_k(\theta)}$ is the probability density function of the derived random variable $\mu_k(\theta)$. Observe that differential entropy takes into account the shape as well as the range of $\mu_k(\theta)$. For example, if two reward functions are linear in $\theta$, the one with a higher slope will have higher differential entropy, as we would desire from an informativeness metric. Evaluating the differential entropy in $\mu_k(\theta)$, i.e., , can be computationally challenging. 

Other than the two metrics described above, there might be alternative (and potentially more complicated) ways of quantifying the  informativeness of an arm. Another candidate would be the {\em information gain} metric proposed in \citep{russo2014learning}, which defines informativeness in terms of identifying the best arm, rather than inferring $\theta^*$. However, as already mentioned in \citep{russo2014learning} by the authors, information gain is computationally challenging to implement in practice outside of certain specific class of problems where prior distribution of $\theta$ is Beta or Gaussian.

\subsection{Proposed \textsc{Informative Algorithm}-C and its Expected Regret}
Given an informativeness metric $I_k(\hat{\Theta}_t)$, we define the most informative arm for the confidence set $\hat{\Theta}_t$ as $\arm_{\hat{\Theta}_t} = \arg \max_{\arm \in \mathcal{K}} I_k(\hat{\Theta}_t)$. At round $t$, Informative Algorithm-C (described in \Cref{alg:Informative}) picks the most informative arm $k_{\hat{\Theta}_t}$ with probability $\frac{\gamma}{t^{d}}$ where $d>1$, and otherwise uses UCB-C or TS-C to pull one of the competitive arms.Here, $\gamma$ and $d$ are hyperparameters of the Informative UCB-C algorithm. Larger $\gamma$ or small $d$ results in more exploration during the initial rounds. Setting the probability of pulling the most informative arm as $\frac{\gamma}{t^{d}}$ ensures that the algorithm pulls the informative arms more frequently at the beginning. This helps shrink $\hat{\Theta}_t$ faster. Setting $d>1$ ensures that informative but non-competitive arms are only pulled $\sum_{t = 1}^{\infty} \frac{\gamma}{t^{d}}= \OO(1)$ times in expectation. Thus, asymptotically the algorithm will behave exactly as the underlying Algorithm-C and the regret of Informative-Algorithm-C is at most an $\OO(1)$ constant worse than the Algorithm-C algorithm.

\begin{algorithm}[t]
\hrule 
\vspace{0.1in}
\begin{algorithmic}
\STATE1: Steps 1 to 5 as in \Cref{alg:formalAlgoUCBC}
\STATE2:  \textbf{Identify $k_{\hat{\Theta}_t}$, i.e., the most informative arm for set $\hat{\Theta}_t$}: 
\STATE $\arm_{\hat{\Theta}_t} = \arg \max_{\arm \in \mathcal{K}} I_k(\hat{\Theta}_t)$
\STATE3: \textbf{Play informative arm with probability $\frac{\gamma}{t^{d}}$, play UCB-C otherwise:} 
\STATE $$ \small{\arm_{\slot + 1} = \begin{cases} k_{\hat{\Theta}_t} &\text{w.p.} ~ \frac{\gamma}{t^{d}}, \\ \arg \max_{\arm \in \competitiveArms_\slot} \left( \hat{\mu}_\arm(\slot) + \sqrt{\frac{2 \alpha \sigma^2 \log \slot}{\pulls_\arm(\slot)}} \right) ~ &\text{w.p.} \quad 1 - \frac{\gamma}{t^{d}} \end{cases} }$$ 
\STATE4: Update empirical mean, $\hat{\mu}_\arm$ and $\pulls_\arm$ for arm $\arm_{t+1}$.
\end{algorithmic}
\vspace{0.1in}
\hrule
\caption{Informative UCB-C}
\label{alg:Informative}
\end{algorithm}

\subsection{Simulation results}
We implement two versions of Informative-Algorithm-C, namely \textsc{Algorithm}-C-KLdiv and \textsc{Algorithm}-C-Entropy, which use the KL-divergence and Entropy metrics respectively to identify the most informative arm $k_{\hat{\Theta}_t}$ at round $t$. \textsc{Algorithm}-C-KLdiv picks the arm with highest variance in $\hat{\Theta}_t$, i.e., $I_k(\hat{\Theta}_t) = \arg\max_k V_k(\hat{\Theta}_t)$. \textsc{Algorithm}-C-Entropy picks an arm whose mean reward function, $\mu_k(\theta)$, has largest shannon entropy for $\theta \in \hat{\Theta}_t$ (assuming $\theta$ to be a uniform random variable in $\hat{\Theta}_t$). As a baseline for assessing the effectiveness of the informativeness metrics, we also implement \textsc{Algorithm}-C-Random which selects $k_{\hat{\Theta}_t}$ by sampling one of the arms uniformly at random from the set of all arms $\mathcal{K}$ at round $t$.

\begin{figure}[t]
    \centering
    \includegraphics[width = 0.75\textwidth]{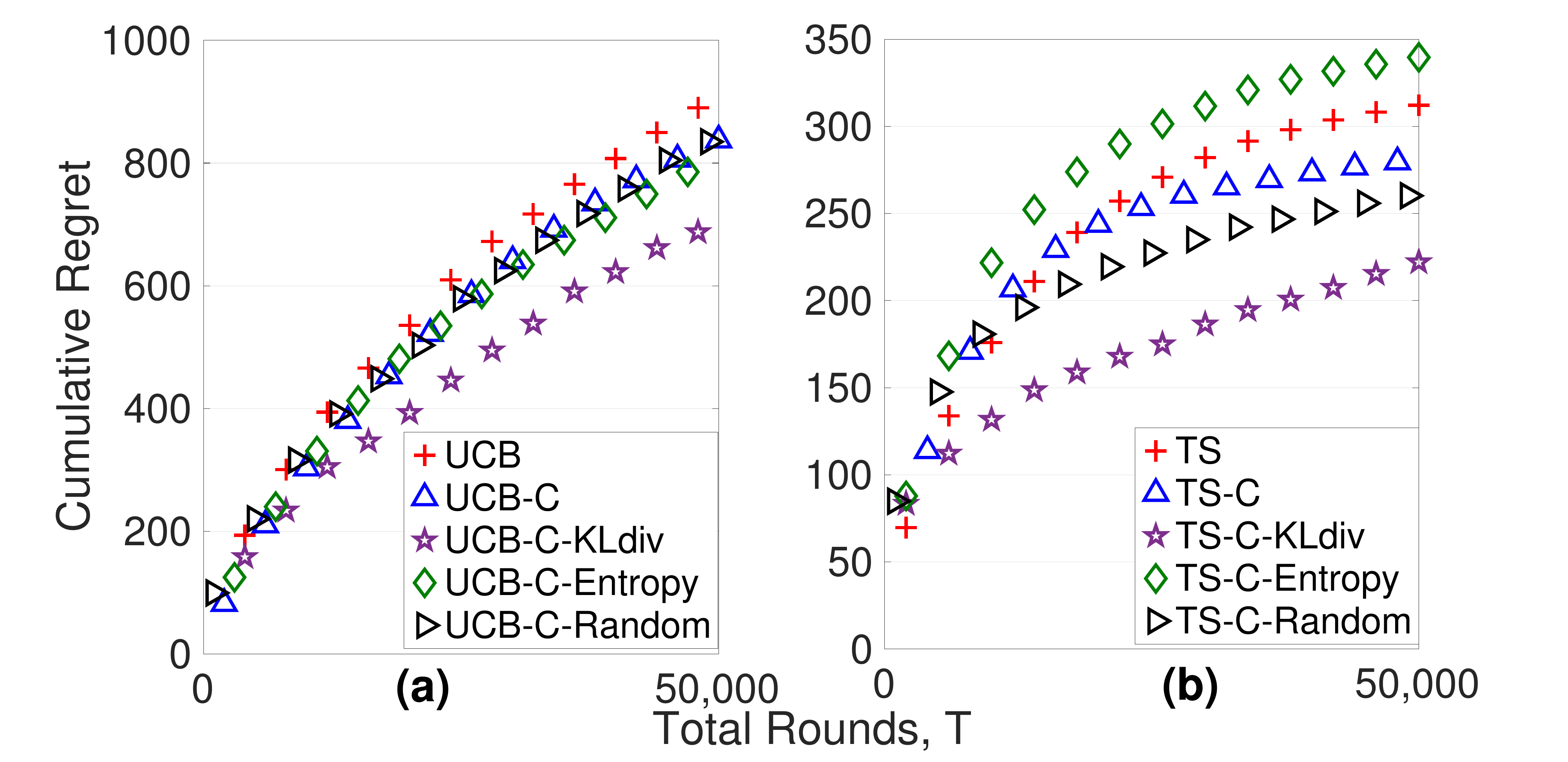}
    \caption{\sl Performance comparison of \textsc{Algorithm}, \textsc{Algorithm}-C and Informative \textsc{Algorithm}-C algorithms (with parameter $\gamma =30$, $d = 1.1$) for the example shown in \Cref{fig:simyo2} with $\theta^* = 3.1$. UCB-C, TS-C do not pull Arm 3 at all, but UCB-C-KLdiv, TS-C-KLdiv pull it in the initial rounds to determine whether $\theta^* > 3$ or not. As a result, UCB-C-KLdiv and TS-C-KLdiv shrink $\hat{\Theta}_t$ faster initially and have a better empirical performance than UCB-C and TS-C, while retaining similar regret guarantees of UCB-C and TS-C respectively.}
    \label{fig:simyo1}
\end{figure}

\Cref{fig:simyo1} shows the cumulative regret of the aforementioned  algorithms for the reward functions shown in \Cref{fig:simyo2}, where the hidden parameter $\theta^* = 3.1$. Among UCB-C, UCB-C-KLdiv, UCB-C-Entropy and UCB-C-Random, UCB-C-KLdiv has the smallest cumulative regret. This is because UCB-C-KLdiv identifies Arm 3 as the most informative arm, samples of which are helpful in identifying whether $\theta^* > 3$ or $\theta^* < 3$. Hence, occasional pulls of Arm 3 lead to fast shrinkage of the set $\hat{\Theta}_t$. In contrast to UCB-C-KLdiv, UCB-C-Entropy identifies Arm 2 as the most informative arm for $\hat{\Theta} = [0,6]$, due to which UCB-C-Entropy samples Arm 2 more often in the initial stages of the algorithm. As the information obtained from Arm 2 is relatively less useful in deciding whether $\theta^* > 3$ or not, we see that UCB-C-Entropy/TS-C-Entropy does not perform as well as UCB-C-KLdiv/TS-C-KLdiv in this scenario. UCB-C-Random picks the most informative arm by selecting an arm uniformly at random from the available set of arms. The additional exploration through random sampling is helpful, but the cumulative regret is larger than UCB-C-KLdiv as UCB-C-Random pulls Arm 3 fewer times relative to UCB-C-KLdiv. For this particular example, cumulative regret of UCB-S was 2500, whereas other UCB style algorithms achieve cumulative regret of 600-800 as shown in the \Cref{fig:simyo1}(a). This is due to the preference of UCB-S to pick Arm 1 in this example. We see similar trends among TS-C, TS-C-KLdiv,  TS-C-Entropy and TS-C-Random. The cumulative regret is smaller for Thompson sampling variants as Thompson sampling is known to outperform UCB empirically.

We would like to highlight that the additional exploration by Informative-Algorithm-C is helpful only in cases where non-competitive arms help significantly shrink the confidence set $\hat{\Theta}_t$. For the experimental setup presented in \Cref{sec:experiments} below, the reward functions are mostly flat as seen in Appendix G and thus, Informative Algorithm-C does not give a significant improvement over the corresponding Algorithm-C. Therefore for clarity of the plots, we do not present experiments results for Informative-C in other settings of this paper.

\section{Experiments with Movielens data}
\label{sec:experiments}
We now show the performance of UCB-C and TS-C on a real-world dataset. We use the {\sc Movielens} dataset \citep{movielenspaper} to demonstrate how UCB-C and TS-C can be deployed in practice and demonstrate their superiority over classical UCB and TS. Since movie recommendations is one of many applications of structured bandits, we do not compare with methods such as collaborative filtering that are specific to recommendation systems. Also, we do not compare with contextual bandits since the structured bandit setting has a different goal of making recommendations \emph{without accessing a user's contextual features}. 

The {\sc Movielens} dataset contains a total of 1M ratings made by $6040$ users for $3883$ movies. There are $106$ different user {\em types} (based on having distinct age and occupation features) and $18$ different genres of movies. The users have given ratings to the movies on a scale of $1$ to $5$. Each movie is associated with one (and in some cases, multiple) genres. For the experiments, of the possibly multiple genres for each movie, we choose one uniformly at random. The set of users that belong to a given type is referred to as  a {\em meta-user}; thus there are $106$ different meta-users. These $106$ different meta-users correspond to the different values that the hidden parameter $\theta$ can take in our setting. For example, one of the meta-users in the data-set represents college students whose age is between $18$ and $24$, and this corresponds to the case $\theta^*=25$. We split the dataset into two equal parts, training and test. This split is done at random, while ensuring that the training dataset has samples from all $106$ meta-users.

For a particular meta-user whose features are unknown (i.e., the true value of  $\theta$ is hidden), we need to sequentially choose one of the genres (i.e., one of the arms) and recommend a movie from that genre to the user. In doing so, our goal is to maximize the {\em total} rating given by this user to the movies we recommended. We use the training dataset ($50\%$ of the whole data) to learn the mean reward mappings from meta-users ($\theta$) to different genres (arms); these mappings are shown in Appendix G. The learned mappings indicate that the mean-reward mappings of meta-users for different genres are related to one another. For example, on average 56+ year old retired users may like documentaries more than children's movies. In our experiments, these dependencies are learned during the training. In practical settings of recommendations or advertising, these mappings can be learned from pilot surveys in which users participate with their consent.

\begin{figure}[t]
    \centering
    \includegraphics[width = 0.65\textwidth]{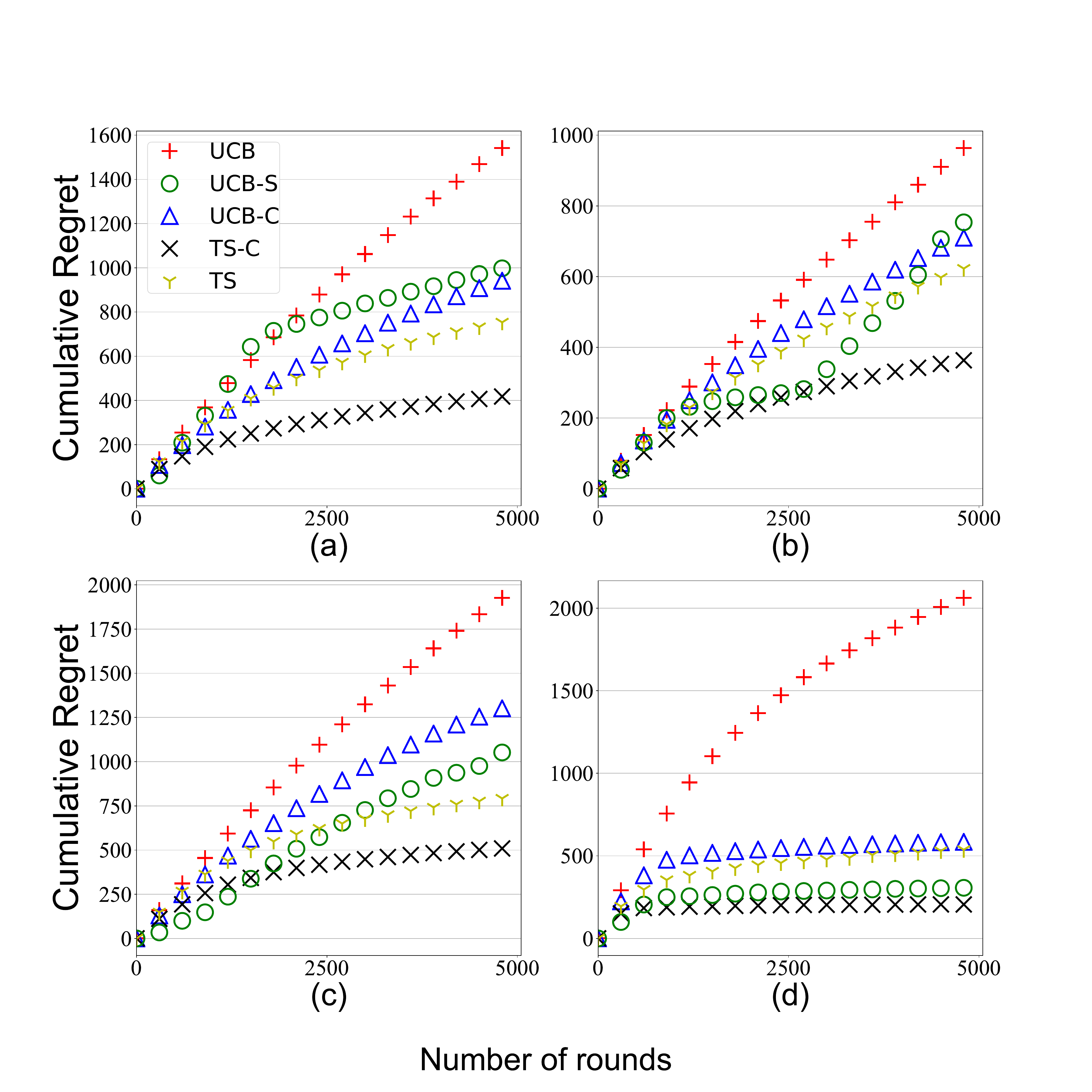}
    \caption{\sl Regret plots for UCB, UCB-S, UCB-C, TS and TS-C for (a) $\theta^* = 67$ (35-44 year old grad/college students), (b) $\theta^* = 87$ (45-49 year old clerical/admin), (c) $\theta^* = 25$ (18-24 year old college students) and (d) $\theta^* = 93$ (56+ Sales and Marketing employees). The value of $C(\theta^*)$ is $6$, $6$, $3$ and $1$ for (a), (b), (c) and (d) respectively -- in all cases $C(\theta^*)$ is much smaller than $K=18$.
}
    \label{fig:theta25}
\end{figure}

\begin{figure}[t]
    \centering
    \includegraphics[width = 0.7\textwidth]{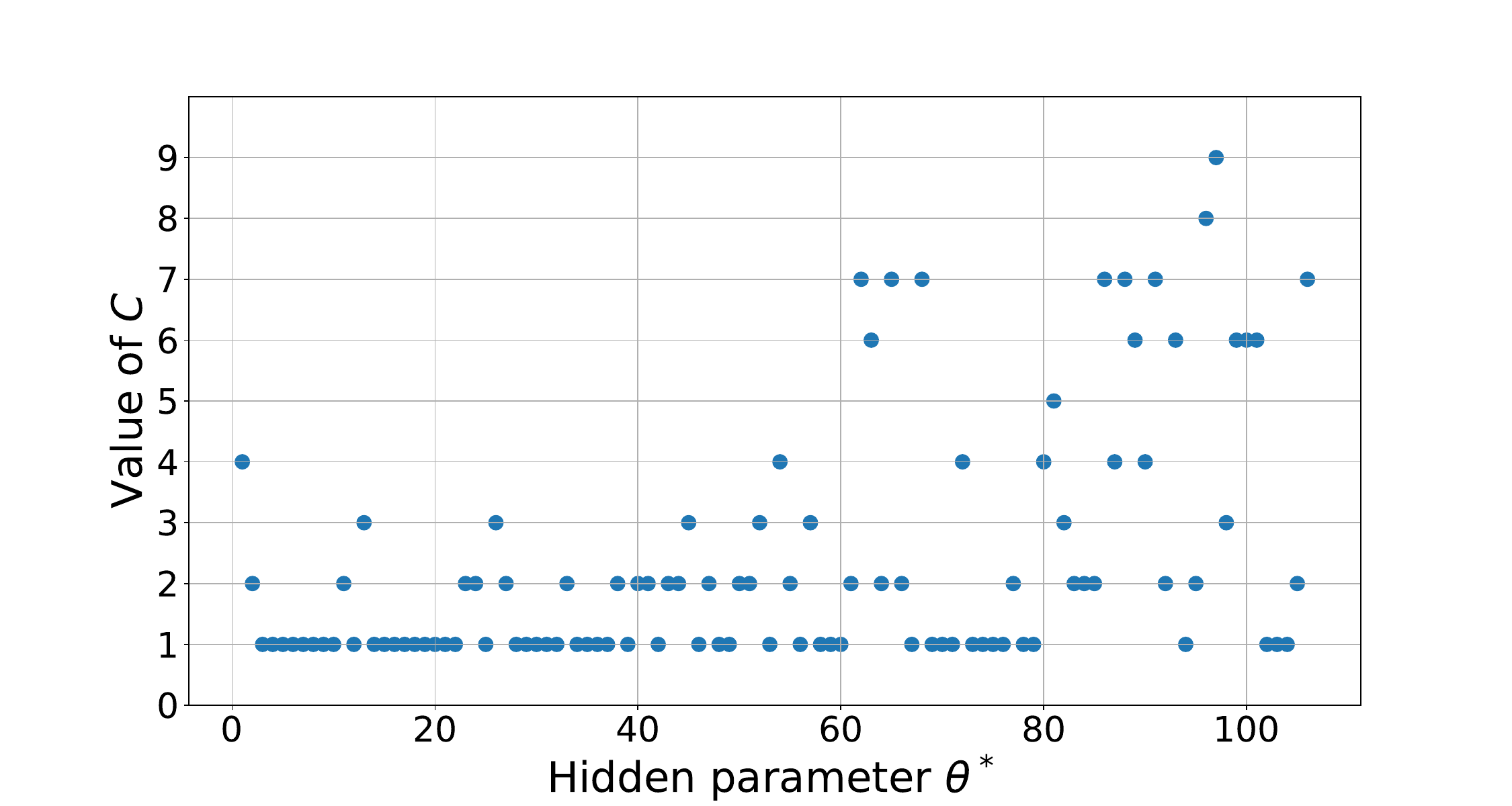}
    \caption{\sl The value of $C(\theta^*)$ varies with the unknown hidden parameter $\theta^*$ (i.e., the age and occupation of the anonymous user). We see that for all $\theta^*$, $C(\theta^*) < K$. While the total number of arms, $K = 18$, the value of $C(\theta^*)$ ranges between $1$ and $9$. This suggests that the $\textsc{Algorithm}$-C approach can lead to significant performance improvement for this problem.}\label{fig:valueofC}
\end{figure}

We test the algorithm for three different meta-users, i.e., for three different values of $\theta^*$. The movie rating samples for these meta-users are obtained from the test dataset, (the remaining $50\%$ of the data). 
\Cref{fig:theta25} shows that UCB-C and TS-C  achieve significantly lower regret than UCB, TS as only a few arms are pulled $\OO(\log T)$ times. This is because only $C(\theta^*) - 1$ of the sub-optimal arms are pulled $\OO(\log T)$ times by our UCB-C and TS-C algorithms. For our experimental setting, the value of $C$ depends on  $\theta^*$ (which is unknown to the algorithm). \Cref{fig:valueofC} shows how $C(\theta^*)$ varies with $\theta^*$, where it is seen  that $C(\theta^*)$ is significantly smaller than $K$ for all $\theta^*$. As a result, the performance improvements observed in 
 \Cref{fig:theta25} for
our UCB-C and TS-C algorithms will apply to other $\theta^*$ values as well. There are $\theta^*$ values for which UCB-C is better than UCB-S, and vice versa. But, TS-C always outperforms UCB-C and UCB-S in our experiments. We tried Informative UCB-C in this setting as well, but the results were similar to that of UCB-C because the arms in this setting are not too informative.

\section{Concluding Remarks}
In this work, we studied a structured bandit problem in which the mean rewards of different arms are related through a common hidden parameter. Our problem setting makes no assumptions on mean reward functions, due to which it subsumes several previously studied frameworks \cite{ata2015global,wang2018regional,mersereau2009structured}. We developed an approach that allows us to extend a classical bandit \textsc{Algorithm} to the structured bandit setting, which we refer to as \textsc{Algorithm-C}. We provide a regret analysis of UCB-C (structured bandit versions of UCB). A key insight from this analysis is that \textsc{Algorithm-C} pulls only $C(\theta^*)-1$ of the $K-1$ sub-optimal arms $\OO(\log T)$ times and all other arms, termed as \textit{non-competitive} arms, are pulled only $\OO(1)$ times. Through experiments on the {\sc Movielens} dataset, we demonstrated that UCB-C and TS-C give significant improvements in regret as compared to previously proposed approaches. Thus, the main implication of this paper is that it provides a unified approach to exploit the structured rewards to drastically reduce exploration in a principled manner.

For cases where non-competitive arms can provide information about $\theta$ that can shrink the confidence set $\hat{\Theta}_t$, we propose a variant of \textsc{Algorithm-C} called informative-\textsc{Algorithm-C} that takes the informativeness of arms into account without increasing unnecessary exploration. Linear bandit algorithms \cite{abbasi2011improved, filippi2010parametric, lattimore2016end} shrink the confidence set $\hat{\Theta}_t$ in a better manner by taking advantage of the linearity of the mean reward functions to estimate $\theta^*$ as the solution to least squares problem \cite{abbasi2011improved}. Moreover, linearity helps them to use self-normalized concentration bound for vector valued martingale, (Theorem 1 in \cite{abbasi2011improved}) to construct the confidence intervals.
Extending this approach to the general structured bandit setting is a non-trivial open question due to the absence of constraints on the nature of mean reward functions $\mu_k(\theta)$. The paper \cite{combes2017minimal} proposes a statistical hypothesis testing method for the case of known conditional reward distributions. Generalizing it to the setting considered in this paper is an open future direction. While we state our results for a scenario where mean reward functions are known, our algorithmic approach, analysis and results can also be extended to a setting where only lower and upper bounds on the mean reward function $\mu_k(\theta)$ are known. This setting is discussed in Appendix A. Another open direction in this field is to study the problem of structured best-arm identification where the goal is to conduct pure exploration and identify the best arm in the fewest number of rounds.

\bibliography{multi_armed_bandit}

\begin{thebibliography}{45}
\providecommand{\natexlab}[1]{#1}
\providecommand{\url}[1]{\texttt{#1}}
\expandafter\ifx\csname urlstyle\endcsname\relax
  \providecommand{\doi}[1]{doi: #1}\else
  \providecommand{\doi}{doi: \begingroup \urlstyle{rm}\Url}\fi

\bibitem[Abbasi-Yadkori et~al.(2011)Abbasi-Yadkori, P{\'a}l, and
  Szepesv{\'a}ri]{abbasi2011improved}
Yasin Abbasi-Yadkori, D{\'a}vid P{\'a}l, and Csaba Szepesv{\'a}ri.
\newblock Improved algorithms for linear stochastic bandits.
\newblock In \emph{Advances in Neural Information Processing Systems}, pages
  2312--2320, 2011.

\bibitem[Agrawal and Goyal(2012)]{agrawal2012analysis}
Shipra Agrawal and Navin Goyal.
\newblock Analysis of thompson sampling for the multi-armed bandit problem.
\newblock In \emph{Conference on Learning Theory}, pages 39--1, 2012.

\bibitem[Agrawal and Goyal(2013{\natexlab{a}})]{agrawal2013further}
Shipra Agrawal and Navin Goyal.
\newblock Further optimal regret bounds for thompson sampling.
\newblock In \emph{Artificial Intelligence and Statistics}, pages 99--107,
  2013{\natexlab{a}}.

\bibitem[Agrawal and Goyal(2013{\natexlab{b}})]{agrawal2013thompson}
Shipra Agrawal and Navin Goyal.
\newblock Thompson sampling for contextual bandits with linear payoffs.
\newblock In \emph{Proceedings of the International Conference on Machine
  Learning (ICML)}, 2013{\natexlab{b}}.
\newblock URL \url{http://dl.acm.org/citation.cfm?id=3042817.3043073}.

\bibitem[Atan et~al.(2015)Atan, Tekin, and van~der Schaar]{ata2015global}
Onur Atan, Cem Tekin, and Mihaela van~der Schaar.
\newblock Global multi-armed bandits with {H{\"o}lder} continuity.
\newblock In \emph{AISTATS}, 2015.

\bibitem[Audibert et~al.(2010)Audibert, Bubeck, and Munosr]{audibert2010best}
Jean-Yves Audibert, Sebastien Bubeck, and R\'emi Munosr.
\newblock Best arm identification in multi-armed bandits.
\newblock In \emph{Proceedings of the annual Conference On Learning Theory
  (COLT)}, pages 359--376, 2010.

\bibitem[Auer et~al.(2002)Auer, Cesa-Bianchi, and Fischer]{auer2002finite}
Peter Auer, Nicolo Cesa-Bianchi, and Paul Fischer.
\newblock Finite-time analysis of the multiarmed bandit problem.
\newblock \emph{Machine learning}, 47\penalty0 (2-3):\penalty0 235--256, 2002.

\bibitem[Bubeck et~al.(2009)Bubeck, Munos, and Stoltz]{bubeck2009pure}
S{\'e}bastien Bubeck, R{\'e}mi Munos, and Gilles Stoltz.
\newblock Pure exploration in multi-armed bandits problems.
\newblock In Ricard Gavald{\`a}, G{\'a}bor Lugosi, Thomas Zeugmann, and Sandra
  Zilles, editors, \emph{Algorithmic Learning Theory}, pages 23--37, Berlin,
  Heidelberg, 2009. Springer Berlin Heidelberg.

\bibitem[Bubeck et~al.(2012)Bubeck, Cesa-Bianchi, et~al.]{bubeck2012regret}
S{\'e}bastien Bubeck, Nicolo Cesa-Bianchi, et~al.
\newblock Regret analysis of stochastic and nonstochastic multi-armed bandit
  problems.
\newblock \emph{Foundations and Trends in Machine Learning}, 5\penalty0
  (1):\penalty0 1--122, 2012.

\bibitem[Chapelle and Li(2011)]{chapelle2011empirical}
Olivier Chapelle and Lihong Li.
\newblock An empirical evaluation of thompson sampling.
\newblock In J.~Shawe-Taylor, R.~S. Zemel, P.~L. Bartlett, F.~Pereira, and
  K.~Q. Weinberger, editors, \emph{Advances in Neural Information Processing
  Systems}, pages 2249--2257, 2011.

\bibitem[Combes et~al.(2017)Combes, Magureanu, and
  Prouti{\`e}re]{combes2017minimal}
Richard Combes, Stefan Magureanu, and Alexandre Prouti{\`e}re.
\newblock Minimal exploration in structured stochastic bandits.
\newblock In \emph{NIPS}, 2017.

\bibitem[Filippi et~al.(2010)Filippi, Cappe, Garivier, and
  Szepesv{\'a}ri]{filippi2010parametric}
Sarah Filippi, Olivier Cappe, Aur{\'e}lien Garivier, and Csaba Szepesv{\'a}ri.
\newblock Parametric bandits: The generalized linear case.
\newblock In \emph{Advances in Neural Information Processing Systems}, pages
  586--594, 2010.

\bibitem[Gabillon et~al.(2012)Gabillon, Ghavamzadeh, and
  Lazaric]{gabillon2012best}
Victor Gabillon, Mohammad Ghavamzadeh, and Alessandro Lazaric.
\newblock Best arm identification: A unified approach to fixed budget and fixed
  confidence.
\newblock In \emph{Advances in Neural Information Processing Systems}, pages
  3212--3220, 2012.

\bibitem[Garivier and Capp{\'e}(2011)]{garivier2011kl}
Aur{\'e}lien Garivier and Olivier Capp{\'e}.
\newblock The kl-ucb algorithm for bounded stochastic bandits and beyond.
\newblock In \emph{Proceedings of the 24th annual Conference On Learning
  Theory}, pages 359--376, 2011.

\bibitem[Garivier and Kaufmann(2016)]{garivier2016optimal}
Aur\'{e}lien Garivier and Emilie Kaufmann.
\newblock Optimal best arm identification with fixed confidence.
\newblock In \emph{Annual Conference on Learning Theory (COLT)}, volume~49 of
  \emph{Proceedings of Machine Learning Research}, pages 998--1027, Columbia
  University, New York, New York, USA, 23--26 Jun 2016. PMLR.
\newblock URL \url{http://proceedings.mlr.press/v49/garivier16a.html}.

\bibitem[Gopalan et~al.(2014)Gopalan, Mannor, and Mansour]{gopalan2014thompson}
Aditya Gopalan, Shie Mannor, and Yishay Mansour.
\newblock Thompson sampling for complex online problems.
\newblock In \emph{Proceedings of the 31st International Conference on Machine
  Learning}, volume~32 of \emph{Proceedings of Machine Learning Research},
  pages 100--108, Bejing, China, 22--24 Jun 2014. PMLR.
\newblock URL \url{http://proceedings.mlr.press/v32/gopalan14.html}.

\bibitem[Graves and Lai(1997)]{graves1997asymptotically}
Todd~L Graves and Tze~Leung Lai.
\newblock Asymptotically efficient adaptive choice of control laws incontrolled
  markov chains.
\newblock \emph{SIAM journal on control and optimization}, 35\penalty0
  (3):\penalty0 715--743, 1997.

\bibitem[Harper and Konstan(2015)]{movielenspaper}
F.~Maxwell Harper and Joseph~A. Konstan.
\newblock The movielens datasets: History and context.
\newblock \emph{ACM Transactions on Interactive Intelligent Systems (TiiS)}, 5,
  4, Article 19, 2015.

\bibitem[Heckel et~al.(2016)Heckel, Shah, Ramchandran, and
  Wainwright]{heckel2016active}
Reinhard Heckel, Nihar~B. Shah, Kannan Ramchandran, and Martin~J. Wainwright.
\newblock Active ranking from pairwise comparisons and the futility of
  parametric assumptions.
\newblock \emph{arXiv}, abs/1606.08842, 2016.
\newblock URL \url{http://arxiv.org/abs/1606.08842}.

\bibitem[Huang et~al.(2017)Huang, Ajallooeian, Szepesv\'{a}ri, and
  M\"{u}ller]{huang2017structured}
Ruitong Huang, Mohammad~M. Ajallooeian, Csaba Szepesv\'{a}ri, and Martin
  M\"{u}ller.
\newblock Structured best arm identification with fixed confidence.
\newblock In \emph{Proceedings of the International Conference on Algorithmic
  Learning Theory (ALT)}, volume~76 of \emph{Proceedings of Machine Learning
  Research}, pages 593--616, Kyoto University, Kyoto, Japan, October 2017.
\newblock URL \url{http://proceedings.mlr.press/v76/huang17a.html}.

\bibitem[Jamieson and Nowak(2014)]{jamieson2014best}
K.~Jamieson and R.~Nowak.
\newblock Best-arm identification algorithms for multi-armed bandits in the
  fixed confidence setting.
\newblock In \emph{Proceedings on the Annual Conference on Information Sciences
  and Systems (CISS)}, pages 1--6, March 2014.

\bibitem[Jamieson et~al.(2014)Jamieson, Malloy, Nowak, and
  Bubeck]{jamieson2014lil}
Kevin Jamieson, Matthew Malloy, Robert Nowak, and S{\'e}bastien Bubeck.
\newblock lil’ucb: An optimal exploration algorithm for multi-armed bandits.
\newblock In \emph{Conference on Learning Theory}, pages 423--439, 2014.

\bibitem[Jamieson et~al.(2015)Jamieson, Jain, Fernandez, Glattard, and
  Nowak]{jamieson2015next}
Kevin~G Jamieson, Lalit Jain, Chris Fernandez, Nicholas~J Glattard, and Rob
  Nowak.
\newblock Next: A system for real-world development, evaluation, and
  application of active learning.
\newblock In \emph{Advances in Neural Information Processing Systems}, pages
  2656--2664, 2015.

\bibitem[Kaufmann et~al.(2016)Kaufmann, Capp{\'e}, and
  Garivier]{kaufmann2016complexity}
Emilie Kaufmann, Olivier Capp{\'e}, and Aur{\'e}lien Garivier.
\newblock On the complexity of best-arm identification in multi-armed bandit
  models.
\newblock \emph{The Journal of Machine Learning Research}, 17\penalty0
  (1):\penalty0 1--42, 2016.

\bibitem[Lai and Robbins(1985)]{lai1985asymptotically}
Tze~Leung Lai and Herbert Robbins.
\newblock Asymptotically efficient adaptive allocation rules.
\newblock \emph{Advances in applied mathematics}, 6\penalty0 (1):\penalty0
  4--22, 1985.

\bibitem[Lattimore and Munos(2014)]{lattimore2014bounded}
Tor Lattimore and R{\'e}mi Munos.
\newblock Bounded regret for finite-armed structured bandits.
\newblock In \emph{Advances in Neural Information Processing Systems}, pages
  550--558, 2014.

\bibitem[Lattimore and Szepesvari(2016)]{lattimore2016end}
Tor Lattimore and Csaba Szepesvari.
\newblock The end of optimism? an asymptotic analysis of finite-armed linear
  bandits.
\newblock \emph{arXiv preprint arXiv:1610.04491}, 2016.

\bibitem[Li et~al.(2010)Li, Chu, Langford, and Schapire]{li2010contextual}
Lihong Li, Wei Chu, John Langford, and Robert~E Schapire.
\newblock A contextual-bandit approach to personalized news article
  recommendation.
\newblock In \emph{Proceedings of the 19th international conference on World
  wide web}, pages 661--670. ACM, 2010.

\bibitem[Li et~al.(2017)Li, Jamieson, DeSalvo, Rostamizadeh, and
  Talwalkar]{li2017hyperband}
Lisha Li, Kevin Jamieson, Giulia DeSalvo, Afshin Rostamizadeh, and Ameet
  Talwalkar.
\newblock Hyperband: A novel bandit-based approach to hyperparameter
  optimization.
\newblock \emph{Journal of Machine Learning Research}, 18\penalty0
  (1):\penalty0 6765--6816, January 2017.
\newblock ISSN 1532-4435.
\newblock URL \url{http://dl.acm.org/citation.cfm?id=3122009.3242042}.

\bibitem[Mannor and Tsitsiklis(2004)]{mannor2004sample}
Shie Mannor and John~N Tsitsiklis.
\newblock The sample complexity of exploration in the multi-armed bandit
  problem.
\newblock \emph{Journal of Machine Learning Research}, 5\penalty0
  (Jun):\penalty0 623--648, 2004.

\bibitem[Mersereau et~al.(2009)Mersereau, Rusmevichientong, and
  Tsitsiklis]{mersereau2009structured}
A.~J. Mersereau, P.~Rusmevichientong, and J.~N. Tsitsiklis.
\newblock A structured multi-armed bandit problem and the greedy policy.
\newblock \emph{IEEE Transactions on Automatic Control}, 54\penalty0
  (12):\penalty0 2787--2802, Dec 2009.
\newblock ISSN 0018-9286.
\newblock \doi{10.1109/TAC.2009.2031725}.

\bibitem[Nino-Mora(2009)]{mora2009stochastic}
J.~Nino-Mora.
\newblock \emph{Stochastic scheduling. In Encyclopedia of Optimization}, pages
  3818--3824.
\newblock Springer, New York, 2 edition, 2009.

\bibitem[Russo and Van~Roy(2014)]{russo2014learning}
Daniel Russo and Benjamin Van~Roy.
\newblock Learning to optimize via posterior sampling.
\newblock \emph{Mathematics of Operations Research}, 39\penalty0 (4):\penalty0
  1221--1243, 2014.

\bibitem[Russo et~al.(2017)Russo, Roy, Kazerouni, and
  Osband]{russo2017tutorial}
Daniel Russo, Benjamin~Van Roy, Abbas Kazerouni, and Ian Osband.
\newblock A tutorial on thompson sampling.
\newblock \emph{CoRR}, abs/1707.02038, 2017.
\newblock URL \url{http://arxiv.org/abs/1707.02038}.

\bibitem[Sen et~al.(2017)Sen, Shanmugam, Dimakis, and
  Shakkottai]{sen2017identifying}
Rajat Sen, Karthikeyan Shanmugam, Alexandros~G Dimakis, and Sanjay Shakkottai.
\newblock Identifying best interventions through online importance sampling.
\newblock \emph{stat}, 1050:\penalty0 9, 2017.

\bibitem[Sen et~al.(2018)Sen, Shanmugam, and Shakkottai]{sen2018contextual}
Rajat Sen, Karthikeyan Shanmugam, and Sanjay Shakkottai.
\newblock Contextual bandits with stochastic experts.
\newblock In \emph{Proceedings of the Twenty-First International Conference on
  Artificial Intelligence and Statistics}, volume~84 of \emph{Proceedings of
  Machine Learning Research}, pages 852--861, 2018.
\newblock URL \url{http://proceedings.mlr.press/v84/sen18a.html}.

\bibitem[Shen et~al.(2018)Shen, Zhou, Tekin, and van~der
  Schaar]{shen2018generalized}
Cong Shen, Ruida Zhou, Cem Tekin, and Mihaela van~der Schaar.
\newblock Generalized global bandit and its application in cellular coverage
  optimization.
\newblock \emph{IEEE Journal of Selected Topics in Signal Processing},
  12\penalty0 (1):\penalty0 218--232, 2018.

\bibitem[Soare et~al.(2014)Soare, Lazaric, and Munos]{soare2014best}
Marta Soare, Alessandro Lazaric, and Remi Munos.
\newblock Best-arm identification in linear bandits.
\newblock In \emph{Advances in Neural Information Processing Systems (NIPS)},
  pages 828--836. 2014.
\newblock URL
  \url{http://papers.nips.cc/paper/5460-best-arm-identification-in-linear-bandits.pdf}.

\bibitem[T{\'a}nczos et~al.(2017)T{\'a}nczos, Nowak, and
  Mankoff]{tanczos2017kl}
Ervin T{\'a}nczos, Robert Nowak, and Bob Mankoff.
\newblock A kl-lucb algorithm for large-scale crowdsourcing.
\newblock In \emph{Advances in Neural Information Processing Systems}, pages
  5894--5903, 2017.

\bibitem[Tao et~al.(2018)Tao, Blanco, and Zhou]{tao2018best}
Chao Tao, Sa{\'u}l Blanco, and Yuan Zhou.
\newblock Best arm identification in linear bandits with linear dimension
  dependency.
\newblock In \emph{Proceedings of the International Conference on Machine
  Learning (ICML)}, volume~80 of \emph{Proceedings of Machine Learning
  Research}, pages 4877--4886, July 2018.
\newblock URL \url{http://proceedings.mlr.press/v80/tao18a.html}.

\bibitem[Tekin and Turgay(2017)]{tekin2017multi}
Cem Tekin and Eralp Turgay.
\newblock Multi-objective contextual multi-armed bandit problem with a dominant
  objective.
\newblock \emph{arXiv preprint arXiv:1708.05655}, 2017.

\bibitem[Thompson(1933)]{thompson1933likelihood}
W.~R. Thompson.
\newblock {On the Likelihood that One Unknown Probability Exceeds Another in
  View of the Evidence of Two Samples}.
\newblock \emph{Biometrika}, 25\penalty0 (3-4):\penalty0 285--294, December
  1933.

\bibitem[Villar et~al.(2015)Villar, Bowden, and Wason]{villar2015multi}
Sof{\'\i}a~S Villar, Jack Bowden, and James Wason.
\newblock Multi-armed bandit models for the optimal design of clinical trials:
  benefits and challenges.
\newblock \emph{Statistical science: a review journal of the Institute of
  Mathematical Statistics}, 30\penalty0 (2):\penalty0 199, 2015.

\bibitem[Wang et~al.(2018)Wang, Zhou, and Shen]{wang2018regional}
Zhiyang Wang, Ruida Zhou, and Cong Shen.
\newblock Regional multi-armed bandits.
\newblock In \emph{AISTATS}, 2018.

\bibitem[White(2012)]{white2012bandit}
John White.
\newblock \emph{Bandit algorithms for website optimization}.
\newblock " O'Reilly Media, Inc.", 2012.

\end{thebibliography}

\newpage
\onecolumn
\appendix

\section{Noisy mean reward functions}
\label{sec:noisyfun}
Throughout the work we assume that the mean reward functions $\mu_k(\theta)$ are known \emph{exactly}. Consider a scenario where only partial information is known about the mean reward functions $\mu_k(\theta)$. We show in this section that our proposed algorithmic approach is flexible enough to accommodate for such scenarios. The flexibility in our framework can be maintained as long as we have some upper and lower bound on mean reward functions. Suppose the mean reward functions are \emph{noisy}, but it is known that $\mu_k^{(l)}(\theta) \leq \mu_k(\theta) \leq \mu_k^{(u)}(\theta)$. In such a case, we can alter the definition of $\hat{\Theta}_t$ as follows, 
$$\hat{\Theta}_t = \hat{\Theta}^{(u)}_t \cap \hat{\Theta}^{(l)}_t,$$ 
where $\hat{\Theta}^{(u)}_t$ and $\hat{\Theta}^{(l)}_t$ are defined as follows,
$$\hat{\Theta}^{(u)}_t = \Bigg\{ \theta: \forall{k \in \mathcal{K}}, \quad  \mu_k^{(l)}(\theta) \leq \hat{\mu}_k(t) + \sqrt{\frac{2 \alpha \sigma^2 \log t}{n_k(t)}} \Bigg\},$$
$$\hat{\Theta}^{(l)}_t = \Bigg\{ \theta: \forall{k \in \mathcal{K}}, \quad \mu_k^{(u)}(\theta) \geq \hat{\mu}_k(t) - \sqrt{\frac{2 \alpha \sigma^2 \log t}{n_k(t)}} \Bigg\}.$$

Subsequently, we call an arm $k$ to be $\hat{\Theta}_t-$Non-Competitive 
if $\mu^{(u)}_\arm(\theta) < \max_{\ell \in \setofArms} \mu^{(l)}_\ell(\theta)$ for all $\theta \in \hat{\Theta}_t$. On doing so, our algorithmic approach can be extended to this framework as well. Additionally, we can re-define the notion of non-competitive and competitive arms in this setting as follows. 
\begin{defn}[Non-Competitive and Competitive arms in the noisy setting]
For any $\epsilon > 0$, let
\begin{align}
\Theta^{*(\epsilon)} = \{\theta: \mu_\bestarm^{(l)}(\theta^*) - \epsilon < \mu^{(u)}_\bestarm(\theta), ~~ \mu^{(l)}_\bestarm(\theta) < \mu_\bestarm^{(u)}(\theta^*) + \epsilon \}. \nonumber
\end{align}
An arm $k$ is said to be non-competitive if there exists an $\epsilon > 0$ such that $\mu^{(u)}_\arm(\theta) < \max_{\ell \in \setofArms} \mu^{(l)}_\ell(\theta)$ for all $\theta \in \Theta^{*(\epsilon)}$.  
Otherwise, the arm is said to be competitive; i.e., if for all $\epsilon > 0$, $\exists \theta \in \Theta^{*(\epsilon)}$ such 
that
$\mu^{(u)}_\arm(\theta) \geq \max_{\ell \in \setofArms} \mu^{(l)}_\ell(\theta)$.
The number of competitive arms is denoted by $C(\theta^*)$. 
\end{defn}
After doing so, the analysis can be easily extended to prove the desired regret bounds in this setting. In particular, one can show that $\Pr(\theta^* \notin \hat{\Theta}_t) \leq 2Kt^{1 - \alpha}$, and subsequently all other results of our paper will follow through. As there is lesser information in this newer framework, one can expect the number of competitive arms to be greater than the number of competitive arms in the setting where $\mu_k^{(l)}(\theta) = \mu_k^{(u)}(\theta)$. As the algorithm, results and analysis can extend beyond the framework described in the paper, it makes our proposed approach even more promising.

\section{Lower bound: Proof for Proposition 1}
\noindent
\textbf{Proof for Proposition 1}
We use the following result of \citep{combes2017minimal} to state  Proposition 1.
\begin{thm}[Lower bound, Theorem 1 in \cite{combes2017minimal}.]
For any uniformly good algorithm \citep{lai1985asymptotically}, and for any $\theta \in \Theta$, we have:
$$\lim \inf_{\totalPulls \rightarrow \infty} \frac{Reg(\totalPulls)}{\log \totalPulls} \geq L(\theta),$$
where $L(\theta)$ is the solution of the optimization problem: 
\begin{align}
    &\min_{\eta(\arm) \geq 0, \arm \in \setofArms} \sum_{\arm \in \setofArms} \eta(\arm)\left(\max_{\ell \in \setofArms}\mu_\ell(\theta) - \mu_\arm(\theta)\right) \nonumber\\
    &\text{subject to} \sum_{\arm \in \setofArms} \eta(\arm) D(\theta,\lambda,\arm) \geq 1, \forall \lambda \in \Lambda(\Theta), \label{optProblem}\\
    where \quad &\Lambda(\theta) = \{\lambda \in \Theta^* : \bestarm \neq \arg \max_{\arm \in \setofArms} \mu_\arm(\lambda) \}. \nonumber
\end{align}
Here, $D(\theta,\lambda,\arm)$ is the KL-Divergence between distributions $f_R(R_{\arm}|\theta,\arm)$ and $f_R(R_{\arm}|\lambda,\arm)$. An algorithm, $\pi$, is uniformly good if $Reg^\pi(\totalPulls,\theta) = \oo(T^a)$ for all $a > 0$ and all $\theta \in \Theta$.
\label{thm:lowerBound}
\end{thm}

We see that the solution to the optimization problem \eqref{optProblem} is $L(\theta) = 0$ only when the set $\Lambda(\theta)$ is empty. The set $\Lambda(\theta)$ being empty corresponds to a case where all sub-optimal arms are $\Theta^*$-non-competitive. This implies that sub-logarithmic regret is possible only when $\tilde{C}(\theta^*) = 1$, i.e there is only one $\Theta^*$-competitive arm, which is the optimal arm, and all other arms are non-competitive. It is assumed that reward distribution of an arm $k$ is parameterized by the mean $\mu_k$ of arm $k$; this ensures that if $\mu_k(\theta) = \mu_k(\lambda)$ then we have  $D(\theta,\lambda,k) = 0$.

\section{Results valid for any \textsc{Algorithm}-C}
We now present results that depend only on Step 1 and Step 2 of our algorithm and hence are valid for any \textsc{Algorithm}-C, where $\textsc{Algorithm}$ can be UCB, Thompson sampling or any other method designed for classical multi-armed bandits with independent arms.

\begin{fact}[Hoeffding's inequality]
Let $Z_1, Z_2, \ldots Z_T$ be i.i.d. random variables, where $Z_i$ is $\sigma^2$ sub-gaussian with mean $\mu$, then 
$$\Pr(\mid \hat{\mu} - \mu \mid) \geq \epsilon) \leq 2 \exp \left(-\frac{\epsilon^2 T}{2 \sigma^2}\right),$$
Here $\hat{\mu}$ is the empirical mean of the $Z_1, Z_2 \ldots Z_T$. 
\label{hoeffding}
\end{fact}

\begin{lem}[Standard result used in bandit literature (Used in Theorem 2.1 of \cite{auer2002finite})]
If $\hat{\mu}_{k,n_k(t)}$ denotes the empirical mean of arm $k$ by pulling arm $k$ $n_k(t)$ times through any algorithm and $\mu_k$ denotes the mean reward of arm $k$, then we have 
\begin{align}
\Pr(|\hat{\mu}_{k,n_k(t)} - \mu_k| \geq \epsilon, \tau_2 \geq &n_k(t) \geq \tau_1 ) \leq \sum_{s = \tau_1}^{\tau_2}2\exp \left(- \frac{ s \epsilon^2}{2 \sigma^2}\right) \nonumber
\end{align}
\label{lem:UnionBoundTrickInt}
\end{lem}

\begin{proof}
Let $Z_1, Z_2, ... Z_t$ be the reward samples of arm $k$ drawn separately. If the algorithm chooses to play arm $k$ for $m^{th}$ time, then it observes reward $Z_m$. Then the probability of observing the event $(\hat{\mu}_{k,n_k(t)} - \mu_k \geq \epsilon, \tau_2 \geq n_k(t) \geq \tau_1)$ can be upper bounded as follows,
\begin{align}
    &\Pr\left(\hat{\mu}_{k,n_k(t)} - \mu_k \geq \epsilon, \tau_2 \geq n_k(t) \geq \tau_1 \right) = \nonumber \\
    & \Pr\left( \left( \frac{\sum_{i=1}^{n_k(t)}Z_i}{n_k(t)} - \mu_k \geq \epsilon \right), \tau_2 \geq n_k(t) \geq \tau_1 \right) \nonumber \\
    &\leq \Pr\left( \left(\bigcup_{m = \tau_1}^{\tau_2} \Bigg\{ \frac{\sum_{i=1}^{m}Z_i}{m} - \mu_k \geq \epsilon \Bigg\} \right), \tau_2 \geq n_k(t) \geq \tau_1 \right) \label{upperBoundTrick} 
\end{align}
\begin{align}
    &\leq \Pr \left(\bigcup_{m = \tau_1}^{\tau_2} \Bigg\{ \frac{\sum_{i=1}^{m}Z_i}{m} - \mu_k \geq \epsilon \Bigg\} \right) \nonumber \\
    &\leq \sum_{s = \tau_1}^{\tau_2} \Pr \left(\frac{\sum_{i=1}^{s}Z_i}{s} - \mu_k \geq \epsilon\right) \nonumber \\
    &\leq \sum_{s = \tau_1}^{\tau_2}\exp \left( - \frac{ s \epsilon^2}{2 \sigma^2}\right).
\end{align}
\end{proof}

\begin{lem}

The probability that the difference between the true mean of arm $k$ and its empirical mean after $t$ time slots is more than $\sqrt{\frac{2 \alpha \sigma^2 \log t}{n_k(t)}}$ is upper bounded by $2t^{1-\alpha}$, i.e.,
$$\Pr\left(|\mu_\arm(\theta^*) - \estimateMean_\arm| \geq \sqrt{\frac{2\alpha \sigma^2 \log \slot}{\pulls_\arm(\slot)}}\right) \leq 2\slot^{1-\alpha}. $$ 
\label{lem:meanOutside}
\end{lem}
\begin{proof}
See that, 
\begin{align}
    \Pr\left(|\mu_\arm(\theta^*) - \estimateMean_{\arm,\pulls_\arm(\slot)}| \geq \sqrt{\frac{2\alpha \sigma^2 \log \slot}{\pulls_\arm(\slot)}}\right) &\leq \sum_{m = 1}^{\slot} \Pr \left(|\mu_\arm(\theta^*) - \estimateMean_{\arm,m} | \geq \sqrt{\frac{2 \alpha \sigma^2 \log \slot}{m}}\right) \label{unionTrick1} \\
    &\leq \sum_{m = 1}^{\slot} 2\slot^{-\alpha} \label{invokeHoeffding}\\
    &= 2\slot^{1-\alpha}. \nonumber
\end{align}
We have \eqref{unionTrick1} from union bound and is a standard trick to deal with the random variable $\pulls_\arm(\slot)$ as it can take values from $1$ to $\slot$ (\Cref{lem:UnionBoundTrickInt}). The true mean of arm $\arm$ is $\mu_\arm(\theta^*)$. Therefore, if $\hat{\mu}_{\arm,m}$ denotes the empirical mean of arm $\arm$ taken over $m$ pulls of arm $\arm$ then, \eqref{invokeHoeffding} follows from \Cref{hoeffding} with $\epsilon$ in \Cref{hoeffding} being equal to $\sqrt{\frac{2\alpha \sigma^2 \log \slot}{\pulls_\arm(\slot)}}$.
\end{proof}

\begin{lem}
Define $E_1(t)$ to be the event that arm $\bestarm$ is $\Theta_t$-non-competitive for the round $t+1$, then,
$$\Pr(E_1(t)) \leq 2 \numArms \slot^{1-\alpha}.$$
\label{lem:elimination3}
\end{lem}
\begin{proof}
Observe that
\begin{align}
    \Pr(E_1(t)) \leq \Pr(\theta^* \notin \hat{\Theta}_t) &= \Pr\left(\bigcup_{\arm \in \setofArms}| \mu_\arm (\theta^*) - \hat{\mu}_{\arm,\pulls_\arm(\slot)} | \geq \sqrt{\frac{2\alpha \sigma^2 \log t}{\pulls_\arm(\slot)}}\right) \label{equalityByDefn3} 
\end{align}
\begin{align}
    &\leq \sum_{\arm = 1}^{\numArms} \Pr\left(| \mu_\arm (\theta^*) - \hat{\mu}_{\arm,\pulls_\arm(\slot)} | \geq \sqrt{\frac{2\alpha \sigma^2 \log t}{\pulls_\arm(\slot)}}\right) \label{unionB} \\
    &\leq \sum_{\arm = 1}^{\numArms} \sum_{m = 1}^{\slot} \Pr\left(| \mu_\arm (\theta^*) - \hat{\mu}_{\arm,m} | \geq \sqrt{\frac{2\alpha \sigma^2 \log t}{m}}\right) \label{unionTrick3}\\
    &\leq \numArms \sum_{m=1}^{\slot}2 \slot^{-\alpha} \label{fromFirstLemma} \\
    &= 2 \numArms \slot^{1-\alpha}. \nonumber
\end{align}
In order for Arm $k^*$ to be $\hat{\Theta}_t$-non-competitive, we need $\mu_{k^*}(\theta) < \max_{k \in \mathcal{K}} \mu_k(\theta) \quad \forall \theta \in \hat{\Theta}_t$. See that $\theta^* \in \hat{\Theta}_t$ implies that arm $k^*$ is $\hat{\Theta}_t$-Competitive. Therefore, $\theta^* \notin \hat{\Theta}_t$ is a necessary condition for Arm $k^*$ to be $\hat{\Theta}_t$-Non-Competitive. Due to this, we have $\Pr(E_1(t)) \leq \Pr(\theta^* \notin \hat{\Theta}_t)$ in \eqref{equalityByDefn3}. We are using $\hat{\mu}_{\arm,m}$ to denote the empirical mean of rewards from arm $\arm$ obtained from it's $m$ pulls. Here \eqref{equalityByDefn3} follows from definition of confidence set and \eqref{unionB} follows from union bound. We have \eqref{unionTrick3} from union bound and is a standard trick to deal with the random variable $\pulls_\arm(\slot)$ as it can take values from $1$ to $\slot$ (\Cref{lem:UnionBoundTrickInt}). The inequality \eqref{fromFirstLemma} follows from Hoeffding's lemma.
\end{proof}

\begin{lem}
Consider a suboptimal arm $\arm \neq \bestarm$, which is $\Theta^{*(\epsilon_\arm)}$-non-competitive. If $\epsilon_\arm \geq \sqrt{\frac{8 \alpha \sigma^2 \numArms \log \slot_0}{\slot_0}}$ for some constant $\slot_0 > 0$, then, 
$$\Pr(\arm_{\slot + 1} = \arm , \bestarm = \maxArm) \leq 2\slot^{1-\alpha},$$
where $k^{\text{max}} = \arg \max_{k \in \mathcal{K}} n_k(t)$.
\label{lem:maxIsBest}
\end{lem}

\begin{proof}
We now bound this probability as, 
\begin{align}
    &\Pr(\arm_{\slot + 1} = \arm , \bestarm = \maxArm) \nonumber \\
    &= \Pr(\arm \in \competitiveArms_\slot, \Index_\arm = \max_{\ell \in \competitiveArms} \Index_\ell , \bestarm = \maxArm) \nonumber \\
    &\leq \Pr(\arm \in \competitiveArms_\slot , \bestarm = \maxArm) \nonumber \\
    &\leq \Pr(|\hat{\mu}_\bestarm - \mu_\bestarm(\theta^*)| > \frac{\epsilon_\arm}{2} , \bestarm = \maxArm) \label{competitiveCondition}\\
    &\leq 2 \slot \exp \left(-\frac{\epsilon^2 \slot}{8 \numArms \sigma^2}\right) \label{usedHoeffding}\\
    &\leq 2 \slot^{1-\alpha} \quad \forall \slot > \slot_0.\label{defntnot}
\end{align}
See that $|\hat{\mu}_\maxArm - \mu_\maxArm | < \frac{\epsilon_\arm}{2} \Rightarrow |\mu_\maxArm(\theta) - \mu_\maxArm(\theta^*)| < \epsilon$ for $\theta \in \tilde{\Theta}_t$. This holds as $\sqrt{\frac{2 \alpha \sigma^2 \log \slot_0}{\slot_0}} \leq \frac{\epsilon_\arm}{2}$ and if $\theta \in {\tilde{\Theta}_t}$, then $|\mu_\maxArm(\theta) - \hat{\mu}_\maxArm| \leq \sqrt{\frac{2 \alpha \sigma^2 \log \slot_0}{\slot_0}} \leq \frac{\epsilon_\arm}{2}$. Therefore in order for arm $\arm$ to be $\Theta_t$-competitive, we need at least $|\hat{\mu}_\bestarm - \mu_\bestarm(\theta^*)| > \epsilon_\arm/2$, which leads to \eqref{competitiveCondition} as arm $\arm$ is $\epsilon_\arm$ non-competitive. Inequality \eqref{usedHoeffding} follows from Hoeffding's inequality. The term $\slot$ before the exponent in \eqref{usedHoeffding} arises as the random variable $\pulls_\bestarm$ can take values from $\frac{\slot}{\numArms}$ to $\slot$ (\Cref{lem:UnionBoundTrickInt}). 
\end{proof}

\section{Unified Regret Analysis}
In this section, we show a sketch of how regret analysis of \textsc{Algorithm-C} can be performed, where \textsc{Algorithm} is a bandit algorithm that works for classical multi-armed bandit problem. In later section, we provide rigorous proof for the regret analysis of the UCB-C algorithm.

\noindent
\textbf{Bound on expected number of pulls for Non-Competitive arms}

For non-competitive arms, we show that the expected number of pulls is $\OO(1)$. Intuition behind the proof of the result is that non-competitive arms can be identified as sub-optimal based on \textit{enough} pulls of the optimal arm, due to which pulling a sub-optimal arm $\OO(\log T)$ becomes unnecessary and it ends up being pulled only $\OO(1)$ times. We now provide a mathematical proof sketch of this idea: 

We bound $\E{\pulls_\arm(\slot)}$ as 
\begin{align}
\E{\pulls_\arm(\totalPulls)} &= \E{\sum_{\slot = 1}^{\totalPulls}\indicator_{\{\arm_\slot = \arm\}}}\\
&= \sum_{\slot = 0}^{\totalPulls-1} \Pr(\arm_{\slot+1} = \arm) \nonumber \\
&= \sum_{\slot = 1}^{\numArms \slot_0} \Pr(\arm_\slot = \arm) + \sum_{\slot = \numArms \slot_0}^{\totalPulls-1} \Pr(\arm_{\slot+1} = \arm) \nonumber \\
&\leq \numArms \slot_0 +  \sum_{\slot = \numArms \slot_0}^{\totalPulls-1} \Pr(\arm_{\slot+1} = \arm , \pulls_{\arm^*}(\slot) = \max_{\arm'} \pulls_{\arm'}(\slot))  + \nonumber\\
&\quad \sum_{\slot = \numArms \slot_0}^{\totalPulls-1} \sum_{\arm' \neq \arm^*} \bigg( \Pr(\pulls_{\arm'}(\slot) = \max_{\arm''} \pulls_{\arm''}(\slot)) \times \Pr(\arm_{\slot+1} = \arm |  \pulls_{\arm'}(\slot) = \max_{\arm''} \pulls_{\arm''}(\slot))\bigg) \nonumber \\
&\leq \numArms \slot_0 + \sum_{\slot = \numArms \slot_0}^{\totalPulls-1} \Pr(\arm_{\slot+1} = \arm , \pulls_{\arm^*}(\slot) = \max_{\arm'} \pulls_{\arm'}(\slot)) + \nonumber \\
& \sum_{\slot = \numArms \slot_0}^{\totalPulls-1} \sum_{\arm' \neq \arm^*} \Pr(\pulls_{\arm'}(\slot) = \max_{\arm''} \pulls_{\arm''}(\slot)) \nonumber \\
&\leq \numArms \slot_0 + \sum_{\slot = \numArms \slot_0}^{\totalPulls - 1} 2t^{1-\alpha}  +\sum_{\slot = \numArms \slot_0}^{\totalPulls-1} \sum_{\arm' \neq \arm^*} \Pr(\pulls_{\arm'}(\slot) = \max_{\arm''} \pulls_{\arm''}(\slot)) \label{usedCommonLemma}
\end{align}
\begin{align}
&\leq \numArms \slot_0 + \sum_{\slot = \numArms \slot_0}^{\totalPulls - 1} 2t^{1-\alpha} +  \sum_{\slot = \numArms \slot_0}^{\totalPulls} \sum_{\arm' \neq \arm^*} \Pr\left(\pulls_{\arm'}(\slot) \geq \frac{\slot}{\numArms}\right) \quad \forall t > t_0 \nonumber
\end{align}

Here \eqref{usedCommonLemma} follows from plugging the result of \Cref{lem:maxIsBest}.

In order to prove that \textsc{Algorithm-C} achieves bounded regret, we need to show that $\Pr\left(n_{k'}(t) \geq \frac{t}{K}\right) \leq \frac{\eta}{t^{1+\epsilon}}, \quad \text{for } k' \neq k^*,$ for some $\eta, \epsilon > 0$. Intuitively, this means that the probability of selecting a sub-optimal arm more than $\frac{t}{K}$ times decays with the number of rounds. This property should intuitively hold true for any good performing bandit algorithm. We prove this rigorously for \textsc{UCB-C}.

\noindent
\textbf{Bound on expected number of pulls for Competitive arms}

For any suboptimal arm $\arm \neq \arm^*$,
\begin{align}
    \E{\pulls_\arm(\totalPulls)} &\leq \sum_{\slot = 1}^{\totalPulls} \Pr(\arm_\slot = \arm) \nonumber \\
    &= \sum_{\slot = 1}^{\totalPulls} \Pr((k_t = k, E_1(\slot)) \cup (E_1^c(\slot), k_t = k)) \label{stepE15} \\
    &\leq \sum_{\slot = 1}^{\totalPulls} \Pr(E_1(\slot)) + \sum_{\slot = 1}^{\totalPulls}  \Pr(E_1^c(\slot), k_t = k) \nonumber \\
    &= \sum_{\slot = 1}^{\totalPulls} 2\numArms\slot^{1-\alpha} + \sum_{\slot = 1}^{\totalPulls}  \Pr(E_1^c(\slot), k_t = k) \label{eliminatedArmProb1} 
\end{align}

Here \Cref{eliminatedArmProb1} follows from the result of \Cref{lem:elimination3} with $E_1(t)$ being the event that arm $\bestarm$ is $\Theta_t$-non-competitive for the round $t+1$.

See that the event $E^{(c)}_1(t) \cap k_{t} = k$ corresponds to the event that a sub-optimal arm $k$ and optimal arm $k^*$ are both present in the competitive set, and the \textsc{Algorithm} selects the sub-optimal arm $k$. The analysis of this term is exactly equivalent to selecting a sub-optimal arm over optimal arm by \textsc{Algorithm}. This leads to a $\OO(\log T)$ bound on the expected pulls of the competitive arms as classical bandit algorithms pull each arm $\OO(\log T)$ times.

\section{Proof for the UCB-C Algorithm}

\begin{lem}
If $\gap_{min} \geq 4 \sqrt{\frac{\numArms \alpha \sigma^2 \log \slot_0}{\slot_0}}$ for some constant $\slot_0 > 0,$ then, 
$$\Pr(\arm_{\slot + 1} = \arm , \pulls_\arm(\slot) \geq s) \leq (2\numArms + 4) \slot^{1-\alpha} \quad \text{for }  s \geq \frac{\slot}{2 \numArms},$$ 
$\forall \slot > \slot_0,$
where $\arm \neq \bestarm$ is a suboptimal arm.
\label{lem:EnoughPullsAlready3}
\end{lem}
\begin{proof}
The probability that arm $\arm$ is pulled at step $t+1$, given it has been pulled $s$ times can be bounded as follows: 
\begin{align}
    \Pr(\arm_{\slot+1} = \arm , \pulls_\arm(\slot) \geq s) &= \Pr(\Index_\arm(\slot) = \max_{\arm' \in \competitiveArms_\slot} \Index_{\arm'}(\slot) , \pulls_\arm(\slot) \geq s) \nonumber \\
    &\leq \Pr(E_1(\slot) \cup (E_1^c(\slot), \Index_\arm(\slot) > \Index_\bestarm(\slot)) , \pulls_\arm(\slot) \geq s) \nonumber \\
    &\leq \Pr(E_1(\slot) , \pulls_\arm(\slot) \geq s) + \Pr(E_1^c(\slot), \Index_\arm(\slot) > \Index_\bestarm(\slot) , \pulls_\arm \geq s) \label{unionBound1}\\
    &\leq 2 \numArms \slot^{1-\alpha} + \Pr\left(\Index_\arm(\slot) > \Index_\bestarm(\slot) , \pulls_\arm(\slot) \geq s\right)  \label{pastLemma}
\end{align}
Here, \eqref{unionBound1} follows from union bound and \eqref{pastLemma} follows from  \Cref{lem:elimination3}. 
We now bound the second term as, 
\begin{align}
    \Pr(\Index_\arm(\slot) > \Index_{\bestarm}(\slot) , \pulls_\arm(\slot) \geq s) &= \Pr\left(\Index_\arm(\slot) > \Index_{\bestarm}(\slot) , \pulls_\arm(\slot) \geq s, \mu_{\bestarm} \leq \Index_{\bestarm}(\slot) \right) + \nonumber \\
    &\quad \Pr\left(\Index_\arm(\slot) > \Index_{\bestarm}(\slot) , \pulls_\arm(\slot) \geq s | \mu_{\bestarm}> \Index_{\bestarm}(\slot) \right) \times \Pr\left(\mu_{\bestarm} > \Index_{\bestarm}(\slot) \right) \label{conditionTerm} \\
    &\leq  \Pr\left(\Index_\arm(\slot) > \Index_{\bestarm}(\slot) , \pulls_\arm(\slot) \geq s, \mu_{\bestarm}\leq \Index_{\bestarm}(\slot) \right) + \Pr\left(\mu_{\bestarm}> \Index_{\bestarm}(\slot) \right) \label{droppingTerms}\\
    &\leq \Pr\left(\Index_\arm(\slot) > \Index_{\bestarm}(\slot) , \pulls_\arm(\slot) \geq s, \mu_{\bestarm}\leq \Index_{\bestarm}(\slot) \right) + 2\slot^{1-\alpha} \label{usingHoeffdingAgain}\\
    &= \Pr\left(\Index_\arm(\slot) > \mu_{\bestarm} ,  \pulls_\arm(\slot) \geq s\right) + 2\slot^{1-\alpha} \nonumber \\
    &= \Pr\left(\hat{\mu}_\arm + \sqrt{\frac{2 \alpha \sigma^2 \log \slot}{\pulls_\arm(\slot)}} > \mu_{\bestarm} , \pulls_\arm(\slot) \geq s \right) + 2\slot^{1-\alpha} \nonumber \\
    &= \Pr\left( \hat{\mu}_\arm - \mu_\arm(\theta^*) > \gap_\arm - \sqrt{\frac{2 \alpha \sigma^2 \log \slot}{\pulls_\arm(\slot)}} , \pulls_\arm(\slot) \geq s\right) + 2\slot^{1-\alpha} \nonumber \\
    &\leq 2\slot \exp\left(-\frac{s}{2\sigma^2} \left(\gap_\arm - \sqrt{\frac{2 \alpha \sigma^2 \log \slot}{s}}\right)^2\right) + 2\slot^{1-\alpha} \label{eqn:chernoffagain}\\
    &= 2\slot^{1-\alpha}\exp\left(- \frac{s}{2\sigma^2} \left(\gap_\arm^2 - 2 \gap_\arm \sqrt{\frac{2 \alpha \sigma^2 \log \slot}{s}}\right)\right) + 2\slot^{1-\alpha} \nonumber \\
    &= 4 \slot^{1-\alpha} \quad \text{ for all  } \slot > \slot_0. \label{finalCondn}
\end{align}
Equation \eqref{conditionTerm} follows from the fact that $P(A) = P(A|B)P(B) + P(A|B^c)P(B^c)$. Inequality \eqref{droppingTerms} arrives from dropping $P(B)$ and $P(A|B^c)$ in the previous expression. We have \eqref{usingHoeffdingAgain} from \Cref{lem:meanOutside} and the fact that $\Index_\arm(\slot) = \hat{\mu}_\arm + \sqrt{\frac{2 \alpha \sigma^2 \log t}{\pulls_\arm(\slot)}}.$ Inequality \eqref{eqn:chernoffagain} follows from the Hoeffding's inequality and the term $\slot$ before the exponent in \eqref{eqn:chernoffagain} arises as the random variable,$\pulls_\arm(\slot)$, can take values between $s$ and $t$ (\Cref{lem:UnionBoundTrickInt}). Equation \eqref{finalCondn} results from the definition of $\slot_0$ and the fact that $s > \frac{\slot}{2\numArms}$.

Plugging the result of \eqref{finalCondn} in the expression \eqref{pastLemma} completes the proof of \Cref{lem:EnoughPullsAlready3}.

\end{proof}

\begin{lem}
Let $t_0$ be the minimum integer satisfying $\gap_{min} \geq 4\sqrt{\frac{K\alpha\sigma^2 \log \slot_0}{\slot_0}}$ then $\forall \slot > \numArms \slot_0$, and $\forall{k \neq k^*}$, we have, 
$$\Pr\left(\pulls_\arm(\slot) > \frac{\slot}{\numArms}\right) \leq 6 \numArms^2 \left(\frac{\slot}{\numArms}\right)^{2-\alpha}.$$
\label{lem:cantBeMax3}
\end{lem}

\begin{proof}
We expand $\Pr\left(\pulls_\arm(\slot) > \frac{\slot}{\numArms}\right)$ as, 
\begin{align}
    \Pr\left(\pulls_\arm(\slot) \geq \frac{\slot}{\numArms}\right) &= 
    \quad \bigg(\Pr\left( \pulls_{\arm}(\slot) \geq \frac{\slot}{\numArms} \mid \pulls_\arm(\slot - 1) \geq \frac{\slot}{\numArms} \right) \times \Pr\left( \pulls_\arm(\slot - 1) \geq \frac{\slot}{\numArms} \right)\bigg) + \nonumber \\
    &\quad \bigg(\Pr\left(\arm_\slot = \arm \mid \pulls_\arm(\slot - 1) = \frac{\slot}{\numArms} - 1\right) \times  \Pr \left(\pulls_\arm (\slot - 1) = \frac{\pulls}{\numArms} - 1\right) \bigg) \nonumber \\
    &\leq \Pr\left(\pulls_\arm(\slot - 1) \geq \frac{\slot}{\numArms}\right) + \Pr\left(\arm_\slot = \arm , \pulls_\arm(\slot - 1) = \frac{\slot}{\numArms} - 1\right) \nonumber \\
    &\leq \Pr\left(\pulls_\arm(\slot - 1) \geq \frac{\slot}{\numArms}\right) + 6 \numArms (\slot - 1)^{1-\alpha} \quad \forall (\slot - 1) > \slot_0. \label{fromPrevLemma1}
\end{align}
Here \eqref{fromPrevLemma1} follows from \Cref{lem:EnoughPullsAlready3}. 

This gives us that $\forall (\slot - 1) > \slot_0$, we have,  $$\Pr\left(\pulls_\arm(\slot) \geq \frac{\slot}{\numArms}\right) - \Pr\left(\pulls_\arm(\slot - 1) \geq \frac{\slot}{\numArms}\right) \leq 6\numArms(\slot - 1)^{1-\alpha}.$$
Now consider the summation
\begin{equation}
  \sum_{\tau = \frac{\slot}{\numArms}}^{\slot} \Pr\left(\pulls_\arm(\tau) \geq \frac{\slot}{\numArms}\right) - \Pr\left(\pulls_\arm(\tau - 1) \geq \frac{\slot}{\numArms}\right) \leq \sum_{\tau = \frac{\slot}{\numArms}}^{\slot}6 \numArms (\tau - 1)^{1-\alpha}.  \nonumber
\end{equation}
This gives us, 
\begin{equation}
    \Pr\left(\pulls_\arm(\slot) \geq \frac{\slot}{\numArms}\right) - \Pr\left(\pulls_\arm\left(\frac{\slot}{\numArms} - 1\right)\geq \frac{\slot}{\numArms}\right) \leq \sum_{\tau = \frac{\slot}{\numArms}}^{\slot}6 \numArms (\tau - 1)^{1-\alpha}. \nonumber
\end{equation}
Since $\Pr\left(\pulls_\arm\left(\frac{\slot}{\numArms} - 1\right)\geq \frac{\slot}{\numArms}\right)  = 0$, we have, 
\begin{align}
    \Pr\left(\pulls_\arm(\slot) \geq \frac{\slot}{\numArms}\right) &\leq \sum_{\tau = \frac{\slot}{\numArms}}^{\slot}6 \numArms (\tau - 1)^{1-\alpha} \nonumber \\
    &\leq 6\numArms^2 \left(\frac{\slot}{\numArms}\right)^{2 - \alpha} \quad \forall \slot > \numArms \slot_0. \nonumber
\end{align}

\end{proof}

\textbf{Proof of Theorem 2}
We bound $\E{\pulls_\arm(\slot)}$ as 
\begin{align}
\E{\pulls_\arm(\totalPulls)} &= \E{\sum_{\slot = 1}^{\totalPulls}\indicator_{\{\arm_\slot = \arm\}}} \nonumber \\
&= \sum_{\slot = 0}^{\totalPulls-1} \Pr(\arm_{\slot+1} = \arm) \nonumber \\
&= \sum_{\slot = 1}^{\numArms \slot_0} \Pr(\arm_\slot = \arm) + \sum_{\slot = \numArms \slot_0}^{\totalPulls-1} \Pr(\arm_{\slot+1} = \arm) \nonumber \\
&\leq \numArms \slot_0 +  \sum_{\slot = \numArms \slot_0}^{\totalPulls-1} \Pr(\arm_{\slot+1} = \arm , \pulls_{\arm^*}(\slot) = \max_{\arm'} \pulls_{\arm'}(\slot))  + \nonumber\\
&\quad \sum_{\slot = \numArms \slot_0}^{\totalPulls-1} \sum_{\arm' \neq \arm^*} \bigg( \Pr(\pulls_{\arm'}(\slot) = \max_{\arm''} \pulls_{\arm''}(\slot)) \times \Pr(\arm_{\slot+1} = \arm |  \pulls_{\arm'}(\slot) = \max_{\arm''} \pulls_{\arm''}(\slot))\bigg) \nonumber \\
&\leq \numArms \slot_0 + \sum_{\slot = \numArms \slot_0}^{\totalPulls-1} \Pr(\arm_{\slot+1} = \arm , \pulls_{\arm^*}(\slot) = \max_{\arm'} \pulls_{\arm'}(\slot)) + \nonumber\\
&\sum_{\slot = \numArms \slot_0}^{\totalPulls-1} \sum_{\arm' \neq \arm^*} \Pr(\pulls_{\arm'}(\slot) = \max_{\arm''} \pulls_{\arm''}(\slot)) \nonumber \\
&\leq \numArms \slot_0 + \sum_{\slot = \numArms \slot_0}^{\totalPulls - 1} 2\slot^{1-\alpha} + \sum_{\slot = \numArms \slot_0}^{\totalPulls} \sum_{\arm' \neq \arm^*} \Pr\left(\pulls_{\arm'}(\slot) \geq \frac{\slot}{\numArms}\right) \label{usingSomeLemma1}\\
&\leq  \numArms \slot_0  + \sum_{\slot = 1}^{\totalPulls} 2 \numArms \slot^{1-\alpha} + \numArms^2 (\numArms - 1) \sum_{\slot = \numArms \slot_0}^{\totalPulls} 6 \left(\frac{\slot}{\numArms}\right)^{2-\alpha}. \label{usingSomeOtherLemma1}
\end{align}
Here, \eqref{usingSomeLemma1} follows from \Cref{lem:maxIsBest} and \eqref{usingSomeOtherLemma1} follows from \Cref{lem:cantBeMax3}.

\textbf{Proof of Theorem 3}
For any sub-optimal arm $\arm \neq \arm^*$,
\begin{align}
    \E{\pulls_\arm(\totalPulls)} &\leq \sum_{\slot = 1}^{\totalPulls} \Pr(\arm_\slot = \arm) \nonumber \\
    &= \sum_{\slot = 1}^{\totalPulls} \Pr((k_t = k, E_1(\slot)) \cup (E_1^c(\slot), k_t = k)) \label{stepE11}\\
    &\leq \sum_{\slot = 1}^{\totalPulls} \Pr(E_1(\slot)) + \sum_{\slot = 1}^{\totalPulls}  \Pr(E_1^c(\slot), k_t = k) \nonumber \\
    & \leq \sum_{\slot = 1}^{\totalPulls} \Pr(E_1(\slot)) + \sum_{\slot = 1}^{\totalPulls}  \Pr(E_1^c(\slot), k_{t} = k, I_k(t-1) > I_{k^*}(t-1)) \nonumber 
\end{align}
\begin{align}
    &\leq \sum_{\slot = 1}^{\totalPulls} \Pr(E_1(\slot)) + \sum_{\slot = 0}^{\totalPulls-1} \Pr(\Index_\arm(\slot)> \Index_{\arm^*}(\slot), k_{t+1} = k) \nonumber \\
    &= \sum_{\slot = 1}^{\totalPulls} 2\numArms\slot^{1-\alpha} + \sum_{\slot = 0}^{\totalPulls-1} \Pr\left(\Index_\arm(\slot) > \Index_{\arm^*}(\slot), k_{t+1} = k \right) \label{eliminatedArmProb4} \\
    &\leq 8 \alpha \sigma^2 \frac{\log (\totalPulls)}{\gap_\arm^2} + \frac{2\alpha}{\alpha - 2} + \sum_{\slot = 1}^{\totalPulls} 2 \numArms \slot^{1-\alpha}. \label{fromAuer1}
\end{align}
Here, \eqref{eliminatedArmProb4} follows from \Cref{lem:elimination3}. We have \eqref{fromAuer1} from the analysis of UCB for the classical bandit problem. This is because the term $\sum_{\slot = 0}^{\totalPulls-1} \Pr\left(\Index_\arm(\slot) > \Index_{\arm^*}(\slot), k_{t+1} = k \right)$ counts the number of times $I_{k}(t) > I_{k^*}(t)$ and $k_{t+1} = k$, which is the exact same term counted in the analysis of the UCB algorithm \citep{auer2002finite} to bound the expected number of pulls of arm $k$. In particular, we can see from analysis of Theorem 1 in \citep{auer2002finite} (equivalently analysis of Theorem 2.1 in \cite{bubeck2012regret}) that $\Pr\left(I_{k}(t) > I_{k^*}(t), k_{t+1} = k, n_k(t) \geq \frac{8\alpha \sigma^2 \log T}{\Delta_k^2}\right) \leq 2t^{1-\alpha}$ and the event $\indicator\{k_{t+1} = k, I_{k}(t) > I_{k^*}(t), n_k(t) \leq 8\alpha\sigma^2\frac{\log T}{\Delta_k^2} \}$ can happen only at-most $8\alpha\sigma^2\frac{\log T}{\Delta_k^2}$ times. \citep{auer2002finite} does this analysis for $[0,1]$ bounded random variables (i.e., $\sigma = 1/2$) and $\alpha = 4$ in their proof of Theorem 1. The analysis is replicated in the proof of Theorem 2.1 in \citep{bubeck2012regret} for a general $\alpha$.
For further details we refer the reader to the analysis of UCB done in Theorem 2.1 of \citep{bubeck2012regret}.

\section{Regret analysis for the TS-C Algorithm}
We now present results for TS-C in the scenario where $K = 2$ and Thompson sampling is employed with Beta priors \citep{agrawal2013further}. In order to prove results for TS-C, we assume that rewards are either $0$ or $1$. The Thompson sampling algorithm with beta prior, maintains a posterior distribution on mean of arm $k$ as $Beta\left(n_k(t) \times \hat{\mu}_k(t) + 1, n_k(t) \times (1 - \hat{\mu}_k(t)) + 1 \right)$. Subsequently, it generates a sample $S_{k}(t) \sim Beta\left(n_k(t) \times \hat{\mu}_k(t) + 1, n_k(t) \times (1 - \hat{\mu}_k(t)) + 1 \right)$ for each arm $k$ and selects the arm $k_{t+1} = \argmax_{k \in \mathcal{K}} S_{k}(t)$. The TS-C algorithm with Beta prior uses this Thompson sampling procedure in its last step, i.e., $k_{t+1} = \argmax_{k \in \mathcal{C}_t} S_{k}(t)$, where $\mathcal{C}_t$ is the set of $\hat{\Theta}_t$-Competitive arms at round $t$. We show that in a 2-armed bandit problem, the regret is $\OO(1)$ if the sub-optimal arm $k$ is non-competitive and is $\OO(\log T)$ otherwise.

For the purpose of regret analysis of TS-C, we define two thresholds, a lower threshold $L_k(\theta^*)$, and an upper threshold $U_k(\theta^*)$ for arm $k\neq k^*$,
\begin{align}
U_k(\theta^*) = \mu_k(\theta^*) + \frac{\Delta_k}{3}, \hspace*{3em} L_k(\theta^*) = \mu_{k^*}(\theta^*) - \frac{\Delta_k}{3}. \label{eq:threshold}
\end{align}

Let $E^{\mu}_{i}(t)$ and $E^{S}_{i}(t)$ be the events that,
\begin{align}
E^{\mu}_k(t) &= \{\hat{\mu}_k(t) \leq U_k(\theta^*) \} \nonumber\\
E^{S}_k(t) &= \{S_k(t) \leq L_k(\theta^*) \}  \label{eq:events}.
\end{align}

To analyse the regret of TS-C, we first show that the number of times arm $k$ is pulled jointly with the event that $n_k(t-1) \geq \frac{t}{2}$ is bounded above by an $\OO(1)$ constant, which is independent of the total number of rounds $T$.

\begin{lem}
\label{lem:notPullsuboptimalifEnough}
If $\Delta_{k} \geq 3 \sqrt{\frac{ \alpha \log t_{0}}{t_{0}}}$ for some constant $t_{0}>0$, then,
\begin{align*}
    \sum_{t = 2t_0}^{T} \Pr\left(k_{t}=k, n_{k}(t-1) \geq \frac{t}{2}\right) = \OO(1)
\end{align*}
where $k \neq k^{*}$ is a sub-optimal arm.
\end{lem}

\begin{proof}
We start by bounding the probability of the pull of $k$-th arm at round $t$ as follows,
\begin{align}
\Pr\left(k_{t}=k, n_{k}(t-1) \geq \frac{t}{2}\right) \overset{(a)}{\leq} & \Pr\left(E_{1}(t), k_{t}=k, n_{k}(t-1) \geq \frac{t}{2}\right) +  \nonumber \\
& \Pr\left(\overline{E_{1}(t)}, k_{t}=k, n_{k}(t-1) \geq \frac{t}{2}\right) \nonumber\\
\overset{(b)}{\leq} & 2Kt^{1-\alpha} +  \Pr\left(\overline{E_{1}(t)}, k_{t}=k, n_{k}(t-1) \geq \frac{t}{2}\right)\nonumber\\
\overset{(c)}{\leq} & 2Kt^{1-\alpha} +  \underbrace{\Pr\left(k_{t} = k, E^{\mu}_k(t), E^{S}_k(t),n_{k}(t-1) \geq \frac{t}{2}\right)}_{\textbf{term A}} + \nonumber\\
 &\underbrace{\Pr\left(k_{t} = k, E^{\mu}_k(t), \overline{E^{S}_k(t)},n_{k}(t-1) \geq \frac{t}{2}\right)}_{\textbf{term B}}+ \nonumber \\
 &\underbrace{\Pr\left(k_{t} = k, \overline{E^{\mu}_k(t)},n_{k}(t-1) \geq \frac{t}{2}\right)}_{\textbf{term C}} 
\label{eq:cric-s} \\
\end{align}
where $(b)$, comes from \Cref{lem:elimination3}. Now we treat each term in \eqref{eq:cric-s} individually. To bound term A, we note that $\Pr\left(k_{t} = k, E^{\mu}_k(t), E^{S}_k(t),n_{k}(t-1) \geq \frac{t}{2}\right) \leq \Pr\left(k_{t} = k, E^{\mu}_k(t), E^{S}_k(t)\right)$. From the analysis in \citep{agrawal2013further} (equation 6), we see that 
$\sum_{t = 1}^{T}\Pr\left(k_{t} = k, E^{\mu}_k(t), E^{S}_k(t)\right) = \OO(1)$ as it is shown through Lemma 2 in \citep{agrawal2013further} that, \\
$\sum_{t = 1}^{T}\Pr\left(k_{t} = k, E^{\mu}_k(t), E^{S}_k(t)\right) \leq \frac{224}{\Delta_k^2} + \sum_{j = 0}^{T} \Theta\left(e^{-\frac{\Delta_k^2j}{18}} + \frac{1}{e^{\frac{\Delta_k^2 j}{36}} - 1} + \frac{9}{(j + 1)\Delta_k^2}e^{-D_k j} \right)$. \\
Here, $D_k = L_k(\theta^*) \log \frac{L_k(\theta^*)}{\mu_{k^*}(\theta^*)} + (1 - L_k(\theta^*)) \log \frac{1 - L_k(\theta^*)}{1 - \mu_{k^*}(\theta^*)}$. 
Due to this, \\ $\sum_{t = 2t_0}^{T} \Pr\left(k_{t} = k, E^{\mu}_k(t), E^{S}_k(t),n_{k}(t-1) \geq \frac{t}{2}\right) = \OO(1)$.
\vspace{2mm}

\noindent
We now bound the sum of term B from $t = 1$ to $T$ by noting that \\
$\Pr\left(k_{t} = k, E^{\mu}_k(t), \overline{E^{S}_k(t)},n_{k}(t-1) \geq \frac{t}{2}\right) \leq \Pr\left(k_{t} = k, \overline{E^{S}_k(t)}\right) $. Additionally, from Lemma 3 in \citep{agrawal2013further}, we get that
$\sum_{t = 1}^{T} \Pr\left(k_{t} = k, \overline{E^{S}_k(t)}\right) \leq \frac{1}{d(U_k(\theta^*),\mu_k(\theta^*))} + 1$, where $d(x,y) = x \log \frac{x}{y} + (1 - x) \log \frac{1 - x}{1 - y}$. As a result, we see that 
$\sum_{t = 1}^{T} \Pr\left(k_{t} = k, E^{\mu}_k(t), \overline{E^{S}_k(t)},n_{k}(t-1) \geq \frac{t}{2}\right) = \OO(1)$.

\vspace{2mm}

\noindent
Finally, for the last term C we can show that,
\begin{align}
 (C) &= \Pr\left(k_{t} = k, \overline{E^{\mu}_k(t)},n_{k}(t-1) \geq \frac{t}{2}\right) \nonumber \\
 & \leq \Pr\left(\overline{E^{\mu}_k(t)},n_{k}(t-1) \geq \frac{t}{2}\right) \nonumber \\
 &= \Pr\left(\hat{\mu}_k - \mu_k > \frac{\Delta_k}{3}, n_k(t-1) \geq \frac{t}{2}\right) \nonumber \\
 &\leq t \exp \left(- \frac{t\Delta_k^2}{9} \right) \label{eq:unbd_and_hoeffding} \\
& \leq t^{1-\alpha} \nonumber
\end{align}
Here \Cref{eq:unbd_and_hoeffding} follows from hoeffding's inequality and the union bound trick to handle random variable $n_k(t-1)$. After plugging these results in \eqref{eq:cric-s}, we get that 

\begin{align}
    \sum_{t = 2t_0}^{T} \Pr\left(k_{t}=k, n_{k}(t-1) \geq \frac{t}{2}\right) &\leq \sum_{t = 2t_0}^{T} 2Kt^{1 - \alpha} + \sum_{t = 2t_0}^{T} \Pr\left(k_{t} = k, E^{\mu}_k(t), E^{S}_k(t),n_{k}(t-1) \geq \frac{t}{2}\right) + \nonumber\\
 & \sum_{t = 2t_0}^{T} \Pr\left(k_{t} = k, E^{\mu}_k(t), \overline{E^{S}_k(t)},n_{k}(t-1) \geq \frac{t}{2}\right)+ \nonumber \\
 & \sum_{t = 2t_0}^{T} \Pr\left(k_{t} = k, \overline{E^{\mu}_k(t)},n_{k}(t-1) \geq \frac{t}{2}\right) \\
 &\leq \sum_{t = 2t_0}^{T} 2Kt^{1 - \alpha} + \OO(1) + \OO(1) + \sum_{t = 2t_0}^{T} t^{1-\alpha} \\
 &= \OO(1)
\end{align}

\end{proof}

We now show that the expected number of pulls by TS-C for a non-competitive arm is bounded above by an $\OO(1)$ constant.\\
\noindent
\textbf{Expected number of pulls by TS-C for a non-competitive arm.}
We bound $\E{\pulls_\arm(\slot)}$ as 
\begin{align}
\E{\pulls_\arm(\totalPulls)} &= \E{\sum_{\slot = 1}^{\totalPulls}\indicator_{\{\arm_\slot = \arm\}}} \nonumber \\
&= \sum_{\slot = 0}^{\totalPulls-1} \Pr(\arm_{\slot+1} = \arm) \nonumber \\
&= \sum_{\slot = 1}^{2\slot_0} \Pr(\arm_\slot = \arm) + \sum_{\slot = 2 \slot_0}^{\totalPulls-1} \Pr(\arm_{\slot+1} = \arm) \nonumber \\
&\leq 2 \slot_0 +  \sum_{\slot = 2 \slot_0}^{\totalPulls-1} \Pr\left(\arm_{\slot+1} = \arm , \pulls_{\arm^*}(\slot) \geq \frac{t}{2}\right)  + \sum_{\slot = 2 \slot_0}^{\totalPulls-1} \Pr\left(\arm_{\slot+1} = \arm , \pulls_{\arm}(\slot) \geq \frac{t}{2}\right) \\
&\leq 2 \slot_0 + \sum_{\slot = 2 \slot_0}^{\totalPulls - 1} 2\slot^{1-\alpha} + \sum_{\slot = 2 \slot_0}^{\totalPulls-1} \Pr\left(\arm_{\slot+1} = \arm , \pulls_{\arm}(\slot) \geq \frac{t}{2}\right) \label{usingSomeLemma}\\
&= \OO(1) \label{eq:lastStepTS}
\end{align}
Here, \eqref{usingSomeLemma} follows from \Cref{lem:maxIsBest} and \eqref{eq:lastStepTS} follows from \Cref{lem:notPullsuboptimalifEnough} and the fact that the sum of $2t^{1-\alpha}$ is bounded for $\alpha > 1$ and $\slot_0  = \inf \bigg\{\tau > 0: \Delta_{\text{min}},\epsilon_k \geq 3 \sqrt{\frac{ \alpha \log \tau}{\tau}} \bigg\}.$

We now show that when the sub-optimal arm $k$ is competitive, the expected pulls of arm $k$ is $\OO(\log T)$.\\

\noindent
\textbf{Expected number of pulls by TS-C for a competitive arm $k \neq k^*$.}:
For any sub-optimal arm $\arm \neq \arm^*$,
\begin{align}
    \E{\pulls_\arm(\totalPulls)} &\leq \sum_{\slot = 1}^{\totalPulls} \Pr(\arm_\slot = \arm) \nonumber \\
    &= \sum_{\slot = 1}^{\totalPulls} \Pr((k_t = k, E_1(\slot)) \cup (E_1^c(\slot), k_t = k)) \label{stepE12} \\
    &\leq \sum_{\slot = 1}^{\totalPulls} \Pr(E_1(\slot)) + \sum_{\slot = 1}^{\totalPulls}  \Pr(E_1^c(\slot), k_t = k) \nonumber \\
    & \leq \sum_{\slot = 1}^{\totalPulls} \Pr(E_1(\slot)) + \sum_{\slot = 1}^{\totalPulls}  \Pr(E_1^c(\slot), k_{t} = k, S_k(t-1) > S_{k^*}(t-1)) \nonumber 
\end{align}
\begin{align}
    &\leq \sum_{\slot = 1}^{\totalPulls} \Pr(E_1(\slot)) + \sum_{\slot = 0}^{\totalPulls-1} \Pr(S_\arm(\slot)> S_{\arm^*}(\slot), k_{t+1} = k) \nonumber \\
    &= \sum_{\slot = 1}^{\totalPulls} 2\numArms\slot^{1-\alpha} + \sum_{\slot = 0}^{\totalPulls-1} \Pr\left(S_\arm(\slot) > S_{\arm^*}(\slot), k_{t+1} = k \right) \label{eliminatedArmProb} \\
    &\leq \frac{9\log(T)}{\Delta_k^2} + \OO(1) + \sum_{\slot = 1}^{\totalPulls} 2 \numArms \slot^{1-\alpha}. \label{fromAuer} \\
    &= \OO(\log T).
\end{align}
Here, \eqref{eliminatedArmProb} follows from \Cref{lem:elimination3}. We have \eqref{fromAuer} from the analysis of Thompson Sampling for the classical bandit problem in \citep{agrawal2013further}. This arises as the term $\Pr\left(S_\arm(\slot) > S_{\arm^*}(\slot), k_{t+1} = k \right)$ counts the number of times $S_k(t) > S_{k^*}(t)$ and $k_{t+1} = k$. This is precisely the term analysed in Theorem 3 of \citep{agrawal2013further} to bound the expected pulls of sub-optimal arms by TS. 
In particular, \citep{agrawal2013further} analyzes the expected number of pull of sub-optimal arm (termed as $\E{k_i(T)}$ in their paper) by evaluating $\sum_{t = 0}^{T-1} \Pr(S_k(t) > S_{k^*}(t), k_{t+1} = k)$ and it is shown in their Section 2.1 (proof of Theorem 1 of \citep{agrawal2013further}) that $\sum_{t = 0}^{T-1} \Pr(S_k(t) > S_{k^*}(t), k_{t+1} = k) \leq \OO(1) + \frac{\log(T)}{d(x_i, y_i)}$. The term $x_i$ is equivalent to $U_k(\theta^*)$ and $y_i$ is equal to $L_k(\theta^*)$ in our notations. Moreover $d(U_k(\theta^*), L_k(\theta^*)) \leq \frac{\Delta_k^2}{9}$, giving us the desired result of \eqref{fromAuer}.

\section{Reward functions for the experiments}

As mentioned in Section 7, we use a part of the Movielens dataset ($50\%$) as the training dataset, on which the mean reward mappings from meta-users to different genres are learned. \Cref{fig:experimentFunction} represents the learned mappings on the training dataset. 

\begin{figure}
    \centering
    \includegraphics[width = 0.96\textwidth]{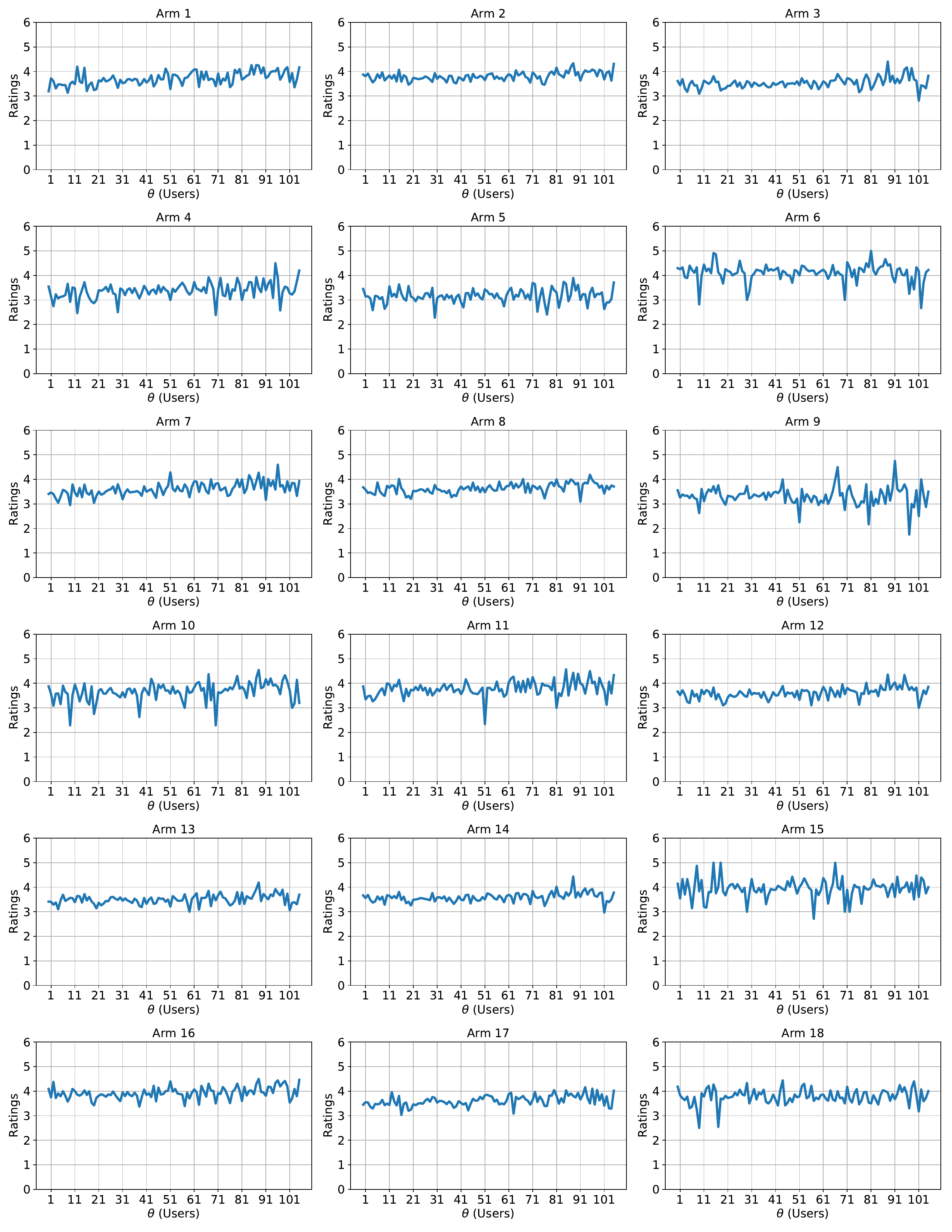}
    \caption{\sl Learned reward mappings from 106 meta-users to each of the movie genres, i.e., the $\mu_k(\theta)$ in the problem setup, with $\theta$ representing different meta-users and $k$(arm) representing different movie genres.}
    \label{fig:experimentFunction}
    \vspace{-0.3cm}
\end{figure}

\end{document}